\documentclass[a4paper,11pt]{article}

\usepackage[utf8]{inputenc}
\usepackage[T1]{fontenc}
\usepackage[english]{babel}
\usepackage{amsmath, amsfonts, amsthm,amssymb,here,dsfont, hyperref}
\usepackage{mathtools}
\usepackage{graphicx,color,subfigure,multirow, here}
\usepackage{natbib}
\usepackage{enumitem}
\usepackage{comment}
\newtheorem{theo}{Theorem}[section]

\newtheorem{prop}[theo]{Proposition}

\newtheorem{coro}[theo]{Corollary}
\newtheorem{lemme}[theo]{Lemma}

\newtheorem{assumption}[theo]{Assumption}
\newcommand{\argmax}[1]{\underset{#1}{\operatorname{arg}\!\operatorname{max}}\;}
\newcommand{\argmin}[1]{\underset{#1}{\operatorname{arg}\!\operatorname{min}}\;}

\newcommand{\w}{\widehat}
\newcommand{\one}{\mathds{1}}
\newcommand{\E}{\mathbb{E}}

\renewcommand{\P}{\mathbb{P}}

\newcommand{\R}{\mathbb{R}}

\newcommand{\p}{{\bf p}}
\newcommand{\h}{{\bf h}}
\newcommand{\bb}{{\bf b}}

\usepackage[textwidth=1.8cm, textsize=tiny]{todonotes}

\title{Empirical risk minimization algorithm

for multiclass classification of S.D.E. paths}

\author{Christophe Denis$^{(1)}$,  Eddy-Michel Ella-Mintsa$^{(2)}$}

\begin{document}
\date{}
\maketitle
\begin{center}
${(1)}$ Université Paris 1 Panthéon-Sorbonne, SAMM\\ 
$(2)$ Université Paris-Saclay, AgroParisTech, INRAE, UMR MIA Paris-Saclay
\end{center}

%%%%%%%%%%%%%%%%%%%%%%%%%%%%%%%%%%%%%%%%%%%%%%%%%%%%%%%%%%%%%%%%%%%%%%%%%%%%%%%%%%%%%%%%%%%%%%%%%%%%%%%%%%%%%%%%%%%%%%%%%%%%%%%%%%%%%%%%%%%%%%%%%%%%%%%%%%%%%%%%%%%%%%%%%%%%%%%%%%%%%%%%%%%%%%%%%%

\begin{abstract}
We address the multiclass classification problem for stochastic diffusion paths, assuming that the classes are distinguished by their drift functions, while the diffusion coefficient remains common across all classes.
In this setting, we propose a classification algorithm that relies on the minimization of the $L_2$ risk. 
We establish rates of convergence for the resulting predictor. Notably, we introduce a margin assumption under which we show that our procedure can achieve fast rates of convergence. Finally, a simulation study highlights the numerical performance of our classification algorithm. 
\end{abstract}

{\bf Keywords.} Multiclass classification, Empirical risk minimization, Diffusion process, $B$-spline basis.

\section{Introduction}
%%%%%%%%%%%%%%%%%%%%%%%%%%%%%%%%%%%%%%%%%%%%%%%%%%%%%%%%%%%%%%%%%%%%%%
%%%%%%%%%%%%%%%%%%%%%%%%%%%%%%%%%%%%%%%%%%%%%%%%%%%%%%%%%%%%%%%%%%%%%%
%%%%%%%%%%%%%%%%%%%%%%%%%%%%%%%%%%%%%%%%%%%%%%%%%%%%%%%%%%%%%%%%%%%%%%
\label{sec:intro-chap3}

Functional data analysis~\citep{ramsaySilverman} is an active field of research, as technological progress has facilitated the large-scale collection of such data in a wide range of fields, including biology~\citep{crow2017introduction}), physics~\citep{romanczuk2012active}), and finance~\citep{lamberton2011introduction}).
%We refer to~\citep{ramsaySilverman} for a general introduction on this topic.
In particular, the classification of functional data is at the core of recent research efforts~\citep{Ismail_Fawaz_2019, NEURIPS2022_times_series}.
A specific case arises when the data are assumed to be generated by diffusion processes.
The literature on statistical methods for stochastic differential equations has generated significant attention.
It encompasses a broad range of statistical problems, such as inference (see for instance~\citep{hoffmann1999adaptive,comte2007penalized, comte2021drift, denis2020ridge, pmlr-v139-kidger21b, marie_2023, MOHAMMADI2024104239}), generative modeling~\citep{pmlr-v37-sohl-dickstein15, YangGenerative}, and supervised learning~\citep{cadre2013supervised, denis2020consistent, gadat2020optimal, denis2024nonparametric}.

In this work, we consider the multi-class classification problem where the observation is a random pair $(\bar{X},Y)$. The label $Y$ belongs to $\{1, \ldots, K\}$
while the feature $\bar{X}$ is assumed to be discrete observations on $[0,1]$ of a time homogeneous diffusion $X$, solution of the following stochastic differentiable equation (s.d.e.) 
\begin{equation}
\label{eq:introTheModel}
dX_t =  b_Y^*(X_t) +\sigma^*(X_t)dW_t,
\end{equation}
where $(W_t)_{t \in [0,1]}$ is a standard Brownian motion.
Hence, in our model, the classes are discriminated by the drift functions while the diffusion coefficient is common for all classes. We assume that both drift functions and diffusion coefficient are unknown. In particular, no parametric form for both drift functions and diffusion coefficient is assumed.
In this framework a classifier $g$ is a measurable function such that $g(\bar{X})$ provides a prediction of $Y$. Its performance is assessed through the misclassification risk. 
%with respect to its oracle counterpart $g^*$ the Bayes classifier that minimizes the misclassification risk.
We assume that the learning sample consists of $N$ independent realizations of $(\bar{X},Y)$. The goal is then to build, based on the learning sample, a classifier $\hat{g}$ such that its risk is closed to the risk of the Bayes classifier $g^*$ that minimizes the misclassification risk. 
We propose an algorithm based on the empirical risk minimization (E.R.M.) principle.
In particular, the resulting predictor minimizes the empirical $L_2$ risk over a set of score functions that relies on modeling both drift and diffusion coefficients. The resulting predictor offers good numerical and theoretical performance.
Notably, we significantly extend the results compared to the existing literature. Specifically, we show that our procedure achieves faster rates of convergence than those obtained in~\citep{denis2024nonparametric} under weaker assumptions. A salient point of our contributions is the introduction of a margin condition~\citep{audibert2007fast} tailored to our specific problem. Based on this assumption, we demonstrate that fast rates of convergence can be achieved.

\subsection{Related works}

There is a wide body of literature that deals with the supervised classification problem for temporal data. For instance, depth classification for functional data~\citep{lopez2006depth, cuevas2007, lange2014, deMicheaux2021},  times series classification~\citep{wang2020, NEURIPS2022_times_series}  or classification of Gaussian processes~\citep{baillo2011, torrecilla2020}.
Recently, \cite{gadat2020optimal} studies a similar, but simpler, problem as defined in Equation~\eqref{eq:introTheModel}, where the drift function is time-dependent and the diffusion coefficient is known. They consider a plug-in approach and provide rates of convergence for the proposed procedure.

Closest to ours, several works deal with the supervised classification problem defined by Equation~\eqref{eq:introTheModel}. In particular, in the binary classification setting, \cite{cadre2013supervised} considers a classification procedure based on empirical risk minimizer that relies on the $0-1$ loss. In order to derive theoretical properties, due to the use of the $0-1$ loss, the set over which the minimization is performed should be finite and sufficiently large to provide a good approximation of the Bayes classifier $g^*$. It is worth noting that it is a significant limitation in applying this method. It is shown that the resulting classifier reaches a rate of convergence of order $N^{-k/(2k+1)}$ over the space $\mathcal{C}^k$ of $k$-continuously differentiable functions. 
However, due to the non convexity of the $0-1$ loss, the proposed predictor can not be considered for practical purpose. A similar result is provided in~\citep{denis2020consistent} for the multi-class framework. In~\citep{denis2020consistent}, the authors consider also a classification procedure based on the minimization of the empirical $L_2$ loss for which they derive rates of convergence. However, they only consider the parametric setting. More precisely, a parametric form for the drift functions is assumed while the diffusion coefficient is assumed to be known and constant. Finally, \cite{denis2024nonparametric} considers a plug-in approach where both  drift functions and  diffusion coefficient are preliminarily estimated. Then the authors use the obtained estimator to build their classification procedure. The construction of their estimator relies also on $B$-spline approximation of both drift and diffusion functions.
They obtain rates of convergence for their procedure under different sets of assumptions. In particular, when the drift and diffusion functions are assumed to be Lipschitz and bounded, the authors show that the resulting classifier achieves a rates of convergence of order $N^{-1/5}$ up to some extra factor of order $\exp(\sqrt{\log(N)})$.

\subsection{Main contribution}

In this work, we propose a classification procedure based on empirical risk minimization, which is more effective than the plug-in approach. Indeed, it directly focuses on the prediction task, 
rather than on the intermediate problem of estimating the drift and diffusion functions.
Due to the limitations of the method considered in~\citep{cadre2013supervised, denis2020consistent} an important question arises: the challenge is to derive rates of convergence for convex surrogates of the $0-1$ loss in the nonparametric setting while ensuring good computational properties. One of the main difficulty is to built a set of classifiers that provides strong approximation properties. Due to the nature of the data, this leads to some technical challenges. Specifically, we propose a set of classifiers that relies on the modeling of drift and diffusion functions by $B$-spline functions~\citep{deboor1978}. 
The resulting procedure exhibits appealing properties from both theoretical and numerical point of view. More precisely, our contribution are as follows:

\begin{enumerate}
    \item[i)] We prove the consistency of our E.R.M. type classifier, and derive a rate of convergence of order $N^{-\beta/(2\beta+1)}$, up to a log-factor, where $\beta$ is the H\"older's smoothness parameter associated to the drift and diffusion functions. In particular, when both drift and diffusion functions are Liptschitz,  our classification procedure achieves a rates of convergence of order, up to some logarithmic factor, $N^{-1/3}$ which is faster than the one obtained in \citep{denis2024nonparametric}. In particular, our result does not require that the drift functions are bounded. Therefore, as expected, we show
    that better rates of convergence can obtained for the E.R.M. procedure under weaker assumptions.
    \item[ii)] In the binary classification setup, in the case where the drift functions are bounded and the diffusion coefficient is constant, we introduce a  margin assumption that relies on a distance between drift functions. Notably, under this assumption and if the drift functions belong to the H\"older space with $\beta \geq 1$, we establish a rate of convergence of order $N^{-4\beta/3(2\beta+1)}$, up to some logarithmic factor. In particular if $\beta \geq 3/2$, our classification procedure achieves fast rates of convergence ({\it e.g.} faster than $N^{-1/2}$). 
    \item[iii)] We propose an adaptive version of our E.R.M. predictor that involves the minimization of a penalized criterion. In particular, we show that the adpative classifier reaches same rates of convergence that the non-adaptive one. 
    \item[iv)] We assess the numerical performance of our classification procedure through synthetic data. We show that, for the considered model, our algorithm compares favorably with respect to the plug-in classifier provided in~\cite{denis2024nonparametric} or depth classifiers~\citep{lopez2006depth}.  
    \end{enumerate}

\subsection{Outline of the paper.}

In Section~\ref{sec:ModelAss}, we present the classification problem as well as the main assumptions. Section~\ref{sec:ERM-procedure} is devoted to the construction of the E.R.M. procedure and the general consistency result while the rates of convergence are presented in Section~\ref{sec:RatesConv}. The adaptive predictor and its theoretical properties are presented in~Section~\ref{sec:AdaptiveERM} and a numerical study over simulated data highlights its performance in Section~\ref{sec:NumericStudy}. Finally, we provide a conclusion and draw some perspectives in Section~\ref{sec:Conclusion}. The proof of our results are relegated to the Appendix.

\paragraph*{Notations.}

For an integer $K \geq 1$, the set $\{1,\ldots, K\}$ is denoted by $[K]$. We define the weighted {\rm softmax} function as follows. For a vector ${\bf x} = (x_1, \ldots, x_K) \in \mathbb{R}^K$ and a vector of probability weights ${\bf p}^* = (p_1, \ldots, p_K)$,
\begin{equation*}
{\rm softmax}^{{\bf p}^*}({\bf x)} =  \left({\rm softmax}^{{\bf p}^*}_1({\bf x)}, \ldots, {\rm softmax}^{{\bf p}^*}_K({\bf x)}\right),
\end{equation*}

with, for $k \in [K]$,
\begin{equation*}
{\rm softmax}^{{\bf p}^*}_k({\bf x)} = \dfrac{p_k \exp(x_k)}{\sum_{i=1}^K p_i\exp(x_i)}.
\end{equation*}

\section{Model and assumptions}
%%%%%%%%%%%%%%%%%%%%%%%%%%%%%%%%%%%%%%%%%%%%%%%%%%%%%%%%%%%%%%%%%%%%%%%%%%%%%%
%%%%%%%%%%%%%%%%%%%%%%%%%%%%%%%%%%%%%%%%%%%%%%%%%%%%%%%%%%%%%%%%%%%%%%%%%%%%%%
%%%%%%%%%%%%%%%%%%%%%%%%%%%%%%%%%%%%%%%%%%%%%%%%%%%%%%%%%%%%%%%%%%%%%%%%%%%%%%
\label{sec:ModelAss}

In this section, we introduce the considered model and the observations in Section~\ref{subsec:MulticlassSetup}. The main assumptions considered throughout the paper are detailed in Section~\ref{subsec:MainAss}. Finally, the convexification of the multiclass classification problem is described in 
Section~\ref{subsec:OptimalClassifier}.

\subsection{Multiclass setup for S.D.E. paths}
%%%%%%%%%%%%%%%%%%%%%%%%%%%%%%%%%%%%%%%%%%%%%%%%%%%%%%%%%%%%%%%%%%%%%%%%%%%%%%
%%%%%%%%%%%%%%%%%%%%%%%%%%%%%%%%%%%%%%%%%%%%%%%%%%%%%%%%%%%%%%%%%%%%%%%%%%%%%%
%%%%%%%%%%%%%%%%%%%%%%%%%%%%%%%%%%%%%%%%%%%%%%%%%%%%%%%%%%%%%%%%%%%%%%%%%%%%%%
\label{subsec:MulticlassSetup}

In this work, we consider the following multiclass classification problem. The feature $X = (X_{t})_{t \in [0,1]}$ is assumed to be solution of the following s.d.e. 
\begin{equation}\label{eq:ERM-model}
    dX_t = b^{*}_{Y}(X_t)dt + \sigma^{*}(X_t)dW_t, ~~ t \in [0,1], ~~ X_0 = 0,
\end{equation}
where $Y \in [K]$ is the associated label, and 
$(W_t)_{t \geq 0}$ is a standard Brownian motion independent of $Y$. 
Throughout the paper, we assume that the functions ${\bf b}^{*} = \left(b_1^*, \ldots, b_K^*\right)$, $\sigma^*$ as well as the distribution of $Y$ are unknown.
In the sequel, the distribution of $Y$ is denoted by ${\bf p}^* = (p_1^*, \ldots, p_K^*)$ with 
$p_k^* = \mathbb{P}(Y=k)$, for $k \in [K]$. 

In the multiclass framework, the goal is to construct a measurable function $g$, namely a classifier, that takes as input a solution $X$ of Equation~\eqref{eq:ERM-model}, and outputs a prediction $g(X) \in [K]$ of its associated label. The performance of such classifier $g$ is then assessed through its misclassification risk $R(g) = \mathbb{P}\left(g(X)\neq Y\right)$.
Naturally at this step, a classifier of interest is the Bayes classifier $g^*$ that satisfies 
\begin{equation*}
g^* \in \argmin{g \in \mathcal{G}} R(g),  
\end{equation*}

where $\mathcal{G}$ is the set of all classifiers. However, since the distribution of $(X,Y)$ is unknown, we cannot compute $g^*$. Therefore, our goal is to build a predictor $\widehat{g}$ that mimics the Bayes classifier.

\paragraph*{Discrete time observations.}

Let $n > 0$. In this work, we assume that a generic observation is a couple $(\bar{X},Y)$, where 
$\bar{X} = (X_{k \Delta_n})_{0 = 1, \ldots, n}$ is a discretized observation of $X$ with $\Delta_n = 1/n$.   The construction of an empirical classifier $\widehat{g}$ is then based on a learning sample $\mathcal{D}_N := \left\{(\bar{X}^j,Y_j), ~~ j = 1, \ldots, N\right\}$
composed of $N$ independent discrete observations of the random pair $(\bar{X},Y)$.
Note that for an observation $\bar{X}$, and a classifier based on $\mathcal{D}_N$, we have that $\widehat{g}(X) = \widehat{g}(\bar{X})$.

\paragraph*{Empirical risk minimizer.}

A natural way to build a predictor $\widehat{g}$ is to consider an Empirical risk minimizer (ERM) (see {\it e.g.} \cite{vapnik1991principles}). More precisely, given a set of classifiers $\mathcal{G}$, we define the ERM classifier as 
\begin{equation}
\label{eq:eqERM01}
\widehat{g} \in \argmin{g \in \widetilde{\mathcal{G}}} \dfrac{1}{N} \sum_{j = 1}^N \one_{\{{\widehat{g}}(\bar{X}^j) \neq Y^j\}}.   
\end{equation}

The theoretical properties of the ERM classifiers are well known. In particular, in our specific framework, \cite{denis2020consistent} establishes the consistency of $\widehat{g}$ in the case where $\sigma^*$ and the distribution of $Y$ are known. Nevertheless, due to computational complexity (see {\it e.g.} \cite{arora1997hardness}), an empirical predictor $\widehat{g}$ defined by~\eqref{eq:eqERM01} cannot be considered for practical purposes. 
It is then usual to consider a convexification of the classification problem (\cite{zhang2004statistical}, \cite{bartlett2006convexity}). In particular, the 0-1 loss is then replace by a convex surrogate. In this work, we consider the $L_2$ loss. In Section~\ref{subsec:OptimalClassifier}, we provide a closed-form expression of the optimal classifier based on the $L_2$-loss.

\subsection{Main assumptions}
%%%%%%%%%%%%%%%%%%%%%%%%%%%%%%%%%%%%%%%%%%%%%%%%%%%%%%%%%%%%%%%%%%%%%%%%%%%%%%
%%%%%%%%%%%%%%%%%%%%%%%%%%%%%%%%%%%%%%%%%%%%%%%%%%%%%%%%%%%%%%%%%%%%%%%%%%%%%%
%%%%%%%%%%%%%%%%%%%%%%%%%%%%%%%%%%%%%%%%%%%%%%%%%%%%%%%%%%%%%%%%%%%%%%%%%%%%%%
\label{subsec:MainAss}

In this section, we gather the main assumptions considered throughout the paper.
The first one ensures that each class may occur with non-zero probability.
\begin{assumption}
we assume that 
$\p_{\min} = \min_{k \in \{1, \ldots,K\}} p_k > 0$.
\end{assumption}

The following assumptions ensure the existence of a diffusion process $X$ solution of Equation~\eqref{eq:ERM-model}, and the existence of its transition density.
%\begin{equation*}
%f_X: (t,x) \in [0,1] \times \mathbb{R} \mapsto f_X(t,x),
%\end{equation*}

\begin{assumption}
\label{ass:RegEll}
\begin{enumerate}
    \item $b^{*}_{i}$ and $\sigma^{*2}$ are $L_0-$Lipshitz functions, where =$L_0 >0$ is a numerical constant.
    \item There exist two constants $c_0,c_1>0$ such that $c_0\leq\sigma^{*}(x)\leq c_1$ for all $x\in\mathbb{R}$.
\end{enumerate}
\end{assumption}

Assumption~\ref{ass:RegEll} guarantees the existence of a unique strong solution $X$ of Equation~\eqref{eq:ERM-model} conditionally to the label $Y \in [K]$ (see \cite{karatzas2014brownian}). Moreover, the unique strong solution $X$ of Equation~\eqref{eq:ERM-model} admits a transition density $(t,x) \in [0,1] \times \mathbb{R} \mapsto f_X(t,x)$. Consequently, we obtain that $X$ admits a moment of any order $q \geq 1$ since from \cite{ella2024nonparametric}, \textit{Lemma 2.2}, we have under Assumption~\ref{ass:RegEll},
\begin{equation*}
    \forall~ q \geq 1, ~~ \mathbb{E}\left[\underset{t \in [0,1]}{\sup}{|X_t|^{q}}\right] < \infty.
\end{equation*}
Note that the transition density $f_X$ of the process $X$ which depends on the label $Y$, can be decomposed as follows
\begin{equation*}
    f_X := \sum_{k=1}^{K}{{p}^{*}_{k}f_{k,X}},
\end{equation*}
where for each $k \in [K]$, $f_{k,X}$ is the transition density of the process $X$ conditional on the event $\{Y = k\}$.

Finally, we assume that the Novikov's criterion is satisfied. In particular, it allows to apply the Girsanov Theorem (see \cite{revuz2013continuous}) which is the main argument to derive a closed-form expression of the Bayes classifier (see~\cite{denis2020consistent, denis2024nonparametric}.

\begin{assumption}[Novikov's condition]
\label{ass:Novikov}
$$\forall k \in [K], \ \ \mathbb{E}\left[\exp\left(\frac{1}{2}\int_{0}^{1}{\frac{b^{*2}_{k}}{\sigma^{*2}}(X_s)ds}\right)\right] < \infty.$$
\end{assumption}

\subsection{Convexification of the problem with the square loss }
%%%%%%%%%%%%%%%%%%%%%%%%%%%%%%%%%%%%%%%%%%%%%%%%%%%%%%%%%%%%%%%%%%%%%%%%%%%%%%
%%%%%%%%%%%%%%%%%%%%%%%%%%%%%%%%%%%%%%%%%%%%%%%%%%%%%%%%%%%%%%%%%%%%%%%%%%%%%%
%%%%%%%%%%%%%%%%%%%%%%%%%%%%%%%%%%%%%%%%%%%%%%%%%%%%%%%%%%%%%%%%%%%%%%%%%%%%%%
\label{subsec:OptimalClassifier}

In this section, we provide a closed-form expression of the oracle classifier based on the $L_2$ loss as well as the ERM classifier based on this loss.
First, we recall the characterization of the Bayes classifier in our framework.

\paragraph*{Bayes classifier.}
For $X$ solution of~\eqref{eq:ERM-model},
we recall that the Bayes classifier is defined as follows 
 \begin{equation*}
        g^{*}(X) = \underset{k \in [K]}{\arg\max}{~\pi^{*}_{k}(X)},
    \end{equation*}
 where for each $k \in [K]$, $\pi^{*}_{k}(X) = \mathbb{P}(Y = k | X)$.
The following proposition provide a useful characterization of the Bayes classifier $g^*$.

\begin{prop}
\label{prop:bayes-classif}
   Under Assumption~\ref{ass:Novikov}, we obtain from \cite{denis2020consistent} that
\begin{equation}\label{eq:conditionalproba}
    \forall ~ k \in [K], ~~ \pi^*_k(X) = {\rm softmax}^{{\bf p}^*}_k\left(\mathbf{F}^*(X)\right), ~~~ {\bf F}^* = \left(F_1^*, \ldots, F_K^*\right)
\end{equation}
with, for each $k \in [K]$,
\begin{equation*}
F^*_k(X):=\int_{0}^{1}{\frac{b_k^{*}}{\sigma^{*2}}(X_s)dX_s}-\frac{1}{2}\int_{0}^{1}{\frac{b^{*2}_{k}}{\sigma^{*2}}(X_s)ds}. 
\end{equation*}
\end{prop}

%As we observe the diffusion paths in discrete time, denote by $\bar{g}^{*}$ the discrete time version of the Bayes classifier $g^{*}$ defined by 
%\begin{equation*}
%    \bar{g}^{*}(X) = \underset{i \in \mathcal{Y}}{\arg\max}{~\bar{\pi}^{*}_{i}(X)},
%\end{equation*}
%where $\bar{\pi}^{*}_{i}(X) = \phi^{*}_{i}(\bar{\mathbf{F}}^{*}(X))$, with $\bar{\mathbf{F}}^{*} = \left(\bar{F}^{*}_{1}, \ldots, \bar{F}^{*}_{K}\right)$, and for each $i \in \mathcal{Y}$,
%\begin{equation*}
%    \bar{F}^{*}_{i} = \sum_{k=0}^{n-1}{\dfrac{b^{*}_{i}}{\sigma^{*2}}(X_{k\Delta_n})(X_{(k+1)\Delta_n} - X_{k\Delta_n})} - \dfrac{\Delta_n}{2}\sum_{k=0}^{n-1}{\dfrac{b^{*2}_{i}}{\sigma^{*2}}(X_{k\Delta_n})}.
%\end{equation*}

\paragraph*{Convexification with $L_2$-loss. }

As exposed in Section~\ref{subsec:MulticlassSetup}, we consider the convexification of our multiclass classification problem.
To this end, we consider $\mathcal{H}$, the convex set of measurable score functions of the solution $X$ defined as 
\begin{equation}\label{eq:ScoreFunction}
        \mathcal{H} = \left\{\mathbf{h} = \left(h_1, \ldots, h_K\right) : \mathcal{X} \longrightarrow \R^{K}\right\}.
    \end{equation}
    
From a score function ${\bf h}$, we consider its associated classifier $g_{\bf h}$ defined as
\begin{equation*}
g_{\bf h}(X) = \argmax{k \in [K]} h^k(X).    
\end{equation*}

Now, for a score function ${\bf h}$, we defined its associated $L_2$-risk as 
\begin{equation*}
R_2(h) = \sum_{k=1}^K \mathbb{E}\left[\left(Z_k-h^k(X)\right)^2\right],
\end{equation*}

with $Z_k = 2\one_{\{Y_j = k\}}-1$. Then, it is not difficult to see that
\begin{equation*}
{\bf h}^* = \argmin{{\bf h} \in \mathcal{H}} R_2({\bf h}),
\end{equation*}

is characterized by $h^{*k}(X) = 2\pi_k^*(X)-1$ for $k \in [K]$.
Therefore, we have that $g_{h^*} = g^*$.
Furthermore, for ${\bf h} \in \mathcal{H}$, the Zhang lemma highlights the link between the excess risk of the classifier $g_{\bf h}$ and the excess risk of the score function ${\bf h}$.

\begin{prop}[\cite{zhang2004statistical}]
\label{prop:RelatingExcessRisks}
Let ${\bf h} \in \mathcal{H}$. The following holds
\begin{align*}
    R(g_{{\bf h}}) - R(g_{\bf h^{*}})\leq\frac{1}{\sqrt{2}}\sqrt{{R_2}({\bf h})-{R_2}(\bf h^{*})}.
\end{align*}
\end{prop}
A straightforward consequence of the above proposition is that consistency {\it w.r.t.} 
the $L_2$-risk implies consistency with respect to the misclassification risk. In the next section, we present our proposed procedure which is based on the empirical risk minimization of the $L_2$-risk rather than the misclassification risk.

\section{E.R.M. Procedure}\label{sec:ERM-procedure}
%%%%%%%%%%%%%%%%%%%%%%%%%%%%%%%%%%%%%%%%%%%%%%%%%%%%%%%%%%%%%%%%%%%%%%%%%%%%%%
%%%%%%%%%%%%%%%%%%%%%%%%%%%%%%%%%%%%%%%%%%%%%%%%%%%%%%%%%%%%%%%%%%%%%%%%%%%%%%
%%%%%%%%%%%%%%%%%%%%%%%%%%%%%%%%%%%%%%%%%%%%%%%%%%%%%%%%%%%%%%%%%%%%%%%%%%%%%%
\label{sec:ERMprocedure}

In this section, we detail the proposed classification algorithm.
The building of the procedure relies on the learning sample $\mathcal{D}_N = \{(\bar{X}^j,Y_j), ~~ j = 1, \ldots, N\}$ composed of $N$ i.i.d. random variables distributed according to the distribution of $(\bar{X},Y)$.
Our E.R.M estimator based on the $L_2$-loss relies on a suitable choice of the set of score functions ${\bf H}$. Once the set of score functions defined, the E.R.M. classifier is then expressed as
\begin{equation*}
\widehat{\bf h} \in \argmin{{\bf h} \in {\bf H}} \widehat{R}_2({\bf h}),   
\end{equation*}

where for each ${\bf h} = (h^1, \ldots,h^K) \in {\bf H}$, $\widehat{R}_2({\bf h})$ is the empirical counterpart of $R_2({\bf h})$ given by
\begin{equation*}
\widehat{R}_2({\bf h}) = \dfrac{1}{N} \sum_{j = 1}^N \sum_{k=1}^K \left(Z_k^j -h^k(\bar{X}^j)\right)^2,    
\end{equation*}

with $Z_k^j = 2\one_{Y_j = k}-1$, $k \in [K]$, $j \in [N]$.
The resulting classifier is then given by $g_{\widehat{\bf h}}$.

In Section~\ref{subsec:score-functions}, we define the set of score functions ${\bf H}$ while the whole procedure is detailed in Section~\ref{subsec:procedure}. Finally, a general consistency result is provided in Section~\ref{subsec:Consistency}.

%In this section, we give a specific description of the set of score functions from which is built the E.R.M. type classification procedure for diffusion paths generated by the unique strong solution $X$ of Model~\eqref{eq:ERM-model}.  

\subsection{Set of score functions}
%%%%%%%%%%%%%%%%%%%%%%%%%%%%%%%%%%%%%%%%%%%%%%%%%%%%%%%%%%%%%%%%%%%%%%%%%%%%%%
%%%%%%%%%%%%%%%%%%%%%%%%%%%%%%%%%%%%%%%%%%%%%%%%%%%%%%%%%%%%%%%%%%%%%%%%%%%%%%
%%%%%%%%%%%%%%%%%%%%%%%%%%%%%%%%%%%%%%%%%%%%%%%%%%%%%%%%%%%%%%%%%%%%%%%%%%%%%%
\label{subsec:score-functions}

The construction of the set ${\bf H}$ relies on the characterization of ${\bf h}^* = \left(h^{*1}, \ldots, h^{*K}\right)$, the optimal score function defined in Section~\ref{subsec:OptimalClassifier}:
\begin{equation*}
h^{*k}(X) = 2 \pi_k^{*}(X) -1, \; \; k \in [K].
\end{equation*}

In view of the expression of $(\pi_k^*)_{k \in [K]}$ provided in Equation~\eqref{eq:conditionalproba}, we aim at considering a set of score functions composed of functions ${\bf h}_{{\bf b}, \sigma^2,{\bf p}} = (h_{{\bf b}, \sigma^2,{\bf p}}^1, \ldots, h_{{\bf b}, \sigma^2,{\bf p}}^K)$ that map $\mathbb{R}^{n+1}$ onto $[-1,1]^K$. More precisely, for each $k \in [K]$
\begin{equation*}
h_{{\bf b}, \sigma^2,{\bf p}}^k(\bar{X}) = 2 \bar{\pi}^k_{\mathbf{b},\sigma^2, {\bf p}}(\bar{X}) -1,    
\end{equation*}
where $\bar{\pi}^k_{\mathbf{b},\sigma^2, {\bf p}}$ is a discrete times approximation of $\pi^*_k$ based on some functions ${\bf b}, \sigma^2$, and ${\bf p}$, a vector of probability weights on $[K]$.
Specifically, 
\begin{equation*}
\bar{\pi}^{k}_{\mathbf{b},\sigma^2, {\bf p}}(\bar{X}) =  {\rm softmax}^{p}_{k}(\bar{\mathbf{{F}}}_{{\bf b}, \sigma^2}(\bar{X})),
\end{equation*}

with $\bar{\mathbf{F}}_{{\bf b}, \sigma^2} = \left(\bar{F}^{1}_{{\bf b}, \sigma^2}, \ldots, \bar{F}^{K}_{{\bf b}, \sigma^2}\right)$, and for each $k \in [K]$,
\begin{equation*}
    \bar{F}_{{\bf b},\sigma^2}^{k}(\bar{X}) = \sum_{k=0}^{n-1}{\dfrac{b_{k}}{\sigma^{2}}(X_{k\Delta_n})(X_{(k+1)\Delta_n} - X_{k\Delta_n})} - \dfrac{\Delta_n}{2}\sum_{k=0}^{n-1}{\dfrac{b_{k}^2}{\sigma^2}}(X_{k\Delta_n}).
\end{equation*}

We highlight that for each $k \in [K]$, $\bar{F}^{k}(\bar{X})$ is a discrete times approximation of $F^*_k(X)$, where the true functions $b_k^*, \sigma^{*2}$ are replaced by the functions $b_k, \sigma^2$.
Therefore, as score functions, we consider functions parameterized by 
functions ${{\bf b}}$ and ${\sigma^2}$ which should be viewed as approximations of the drift and diffusion functions. Hereafter, we present the considered spaces of approximation used to model ${\bf b}^*$ and $\sigma^{*2}$.

\paragraph*{Space of approximations.}

To model the drift and diffusion functions, we consider the $B$-spline space (see {\it e.g.} \cite{gyorfi2006distribution}). Specifically, let $M, D > 0$ two  integers, and 
${\bf u} = (u_{-M}, \ldots, U_{D+M})$ the sequence of regular knots of the interval $[-\log(N),\log(N)]$ defined as
\begin{align*}
&u_{-M}  = \ldots = u_0 = -\log(N),  \;\;  {\rm} \;\; u_D = \ldots = U_{D+M} = \log(N).\\
&\forall l  \in  [D], \;\; u_l = -\log(N) + \dfrac{2l\log(N)}{D}.
\end{align*}

Then, we denote by $\{B_l, l = -M, \ldots, D\}$ the $B$-spline basis of order $M$ defined 
by the knots vector ${\bf u}$.
Finally, the considered spaces of approximation are defined as follows

\begin{align*}
&\mathcal{S}_{D} = \left\{\sum_{\ell = -M}^{D-1} a_{\ell}B_{\ell}, ~~ \sum_{\ell=-M}^{D-1}{a^{2}_{\ell}} \leq (D+M)\log^3(N)\right\},  \\
&\widetilde{\mathcal{S}}_{D} = \left\{x \mapsto \sigma^2(x)  = \widetilde{\sigma}^2(x) \one_{\{\widetilde{\sigma}^2(x)\geq 1/\log(N)\}}   + \frac{1}{\log(N)} \one_{\{\widetilde{\sigma}^2(x)\leq 1/\log(N)\}}, \;\; \widetilde{\sigma}^2 \in \mathcal{S}_D\right\}.
\end{align*}

We highlight that $\mathcal{S}_{D}$ is dedicated to the approximation of the drift functions, while
the set $\tilde{S}_{D}$ is dedicated to the approximation of the diffusion coefficient.
In particular, for $\sigma^2 \in \tilde{S}_{D}$, we have $\sigma^2 > 0$. Note that from the properties of ${\bf B}-$spline functions (see {\it e.g.} \cite{gyorfi2006distribution}), the elements of the approximation space $\mathcal{S}_D$ are $M-1$ times differentiable piece-wise polynomial functions with compact supports. The choice of the constraint space $\mathcal{S}_D$
is motivated by its approximation property (see Proposition~3 in~\citep{denis2024nonparametric}

\paragraph*{Set of score functions. }

Now we are able to define the considered set of score functions. Rather to include the estimation of the weights ${\bf p}^*$, we consider initial estimators $\mathbf{\widehat{p}}$ of it based on a additional labeled dataset $\mathcal{D}_N^{(1)}$. Then we consider the following set of score functions
\begin{equation}
\label{eq:eqSetScoreFun}
\w{\bf H}_{D_1,D_2} = \left\{h_{{\bf b}, \sigma^2, \mathbf{{\widehat{p}}}}, \;\; {\bf b} \in \mathcal{S}_{D_1}^K, \;\; \sigma^2 \in \tilde{\mathcal{S}}_{D_2}\right\},   
\end{equation}

where $D_1$ and $D_2$ are the dimension parameters of the considered spaces of approximation. Note that these parameters have to be chosen beforehand. In particular, the choice of $D_1,D_2$ depends on the size $N$ of the learning sample $\mathcal{D}_N$, and have to be properly chosen to ensure the convergence of the resulting classifier. A data-driven choice of these parameters is discussed in Section~\ref{sec:AdaptiveERM}.

\subsection{Procedure}
%%%%%%%%%%%%%%%%%%%%%%%%%%%%%%%%%%%%%%%%%%%%%%%%%%%%%%%%%%%%%%%%%%%%%%%%%%%%%%
%%%%%%%%%%%%%%%%%%%%%%%%%%%%%%%%%%%%%%%%%%%%%%%%%%%%%%%%%%%%%%%%%%%%%%%%%%%%%%
%%%%%%%%%%%%%%%%%%%%%%%%%%%%%%%%%%%%%%%%%%%%%%%%%%%%%%%%%%%%%%%%%%%%%%%%%%%%%%
\label{subsec:procedure}

In this section, we summarize the overall procedure. As exposed in Section~\ref{subsec:score-functions},
the proposed procedure is two steps. A first step is dedicated to the estimation of the distribution of
$Y$ while in a second step, we compute the ERM estimator over the set of score functions $\w{\bf H}_{D_1,D_2}$.
To this extent, we assume that two independent learning samples are available. For the sake of simplicity, we assume that the two learning samples have the same size $N$.

Since the estimation of ${\bf p}^*$ requires to have access to the output variables, we denote the first sample $\mathcal{Y}_N = \{\widetilde{Y}_1, \ldots, \widetilde{Y}_N\}$ which is composed of i.i.d variables distributed according 
to ${\bf p}^*$.
Then, we define $\boldsymbol{\widehat{p}} = \left(\widehat{p}_1, \ldots, \widehat{p}_K\right)$
as the empirical counterpart of ${\bf p}^*$ with
\begin{equation*}
\forall ~ k \in [K], ~~ \widehat{p}_k = \dfrac{1}{N} \sum_{j=1}^N \tilde{Y}_j.   
\end{equation*}

Note that the use of two independent samples for the estimation of  ${\bf p}^{*}$ is mainly considered for technical reasons. For practical purpose, we can use a single learning sample for the estimation of  ${\bf p}^*$ on one side, and the coefficients ${\bf b}^*$ and $\sigma^{*2}$ on the other side.  

Then, based on $\mathcal{D}_N$, the E.R.M is then define as
$g_{\widehat{\bf h}}$ with
\begin{equation*}
\widehat{{\bf h}} \in \argmin{{\bf h} \in \w{\bf H}_{D_1,D_2}} \widehat{R}_2({\bf h}),
\end{equation*}

and $\w{\bf H}_{D_1,D_2}$ defined by Equation~\eqref{eq:eqSetScoreFun}.
Therefore, in view of the parametrization of the set $\w{\bf H}_{D_1, D_2}$, the ERM score function $\widehat{{\bf h}}$ can be expressed as $\widehat{{\bf h}} = h_{\widehat{{\bf b}}, \widehat{\sigma}^2, \widehat{\mathbf{p}}}$, with 
$\widehat{{\bf b}} \in \mathcal{S}_{D_1}^K, \widehat{\sigma}^2 \in \widetilde{\mathcal{S}}_{D_2}$.
In particular, the estimated functions $\widehat{{\bf b}}, \widehat{\sigma}^2$ can be characterized as 
\begin{equation*}
\left(\widehat{{\bf b}}, \widehat{\sigma}^2\right) \in \argmin{({\bf b}, \sigma^2) \in \mathcal{S}_{D_1}^K \times \widetilde{\mathcal{S}}_{D_2}} \widehat{R}_2\left(h_{{\bf b}, \sigma^2, \mathbf{{\w{p}}}}\right)
\end{equation*}

The resulting classifier is then defined as 
\begin{equation*}
\widehat{g} = g_{\widehat{{\bf h}}}.    
\end{equation*}

\subsection{A general consistency result}
%%%%%%%%%%%%%%%%%%%%%%%%%%%%%%%%%%%%%%%%%%%%%%%%%%%%%%%%%%%%%%%%%%%%%%%%%%%%%%
%%%%%%%%%%%%%%%%%%%%%%%%%%%%%%%%%%%%%%%%%%%%%%%%%%%%%%%%%%%%%%%%%%%%%%%%%%%%%%
%%%%%%%%%%%%%%%%%%%%%%%%%%%%%%%%%%%%%%%%%%%%%%%%%%%%%%%%%%%%%%%%%%%%%%%%%%%%%%
\label{subsec:Consistency}

In this section,  we establish the consistency of the ERM classifier $\widehat{g} = g_{\widehat{\bf h}}$ under the assumptions presented in Section~\ref{subsec:MainAss}. Specifically,
We first provide an upper-bound of the average excess risk of the empirical score function $\widehat{\bf h}$. 
%Naturally, the upper-bound of the classifier $\widehat{g}$ relies on the control of the excess risk related to the empirical score function $\hat{h}$
\begin{theo}
\label{thm:EstimErrorERM}
Assume that $\Delta_n = \mathrm{O}(1/N)$. Under Assumptions \ref{ass:RegEll} and \ref{ass:Novikov}, the following holds:
\begin{equation*}
     \mathbb{E}\left[{R}_2(\widehat{\bf h}) - R_2({\bf h}^{*})\right] \leq C\left(\log^{6}(N)\left(\dfrac{1}{D_1^2}+ \dfrac{1}{D_2^2}\right) + \frac{(D_1+D_2)\log(N)}{N} + \dfrac{1}{n}\right),
\end{equation*}
where $C,c>0$ are numerical constants. 
\end{theo}

The proof of the above result relies on the following strategy. First, we focus on the estimation error of the classifier on the interval $[-\log(N),\log(N])$. Second, the excess risk is controlled on the event where the trajectory is outside $[-\log(N),\log(N)]$. In particular, 
the bound provided in Theorem~\ref{thm:EstimErrorERM} is divided into four terms. The two first terms are respectively the bias and estimation error due to the estimation of the functions ${\bf b}^*$, and $\sigma^{*2}$ on the interval $[-\log(N),\log(N)]$ based on the minimization of the empirical risk.
Finally, the last term is related to the discretization error.

An important consequence of Theorem~\ref{thm:EstimErrorERM} is that it implies the consistency of $\widehat{g}$. Indeed, based on Theorem~\ref{thm:EstimErrorERM}, and the Zhang lemma (Proposition~\ref{prop:RelatingExcessRisks}), we establish the following result.
\begin{coro}
\label{coro:Consistency}
Let Assumptions \ref{ass:RegEll} and \ref{ass:Novikov} be fulfilled. For
$\Delta_n = \mathrm{O}(1/N)$, and $D_1 = D_2:=D_N$ chosen such that
\begin{equation*}
\dfrac{\log^6(N)}{D_N^2} \rightarrow 0, \;\; {\rm and} \;\;\dfrac{D_N\log(N)}{N} \rightarrow 0 ~~ \mathrm{as} ~~ N \rightarrow \infty,
\end{equation*}
it holds that
\begin{equation*}
\mathbb{E}\left[R(\widehat{g})-R(g^*)\right] \rightarrow 0, \;\; {\rm as} \;\;  N \rightarrow 0.
\end{equation*}
\end{coro}

\section{Rates of convergence}
%%%%%%%%%%%%%%%%%%%%%%%%%%%%%%%%%%%%%%%%%%%%%%%%%%%%%%%%%%%%%%%%%%%%%%%%%%%%%%
%%%%%%%%%%%%%%%%%%%%%%%%%%%%%%%%%%%%%%%%%%%%%%%%%%%%%%%%%%%%%%%%%%%%%%%%%%%%%%
%%%%%%%%%%%%%%%%%%%%%%%%%%%%%%%%%%%%%%%%%%%%%%%%%%%%%%%%%%%%%%%%%%%%%%%%%%%%%%
\label{sec:RatesConv}

In this section, we study the rates of convergence of the ERM-type classifier $\widehat{g}$ under additional assumptions on the unknown functions $\mathbf{b}^{*} = \left(b^{*}_{1}, \ldots, b^{*}_{K}\right)$ and $\sigma^{*2}$ that characterize our diffusion model.

Throughout this section, we assume that the functions $b^{*}_{k}, ~ k \in [K]$ and $\sigma^{*2}$ belong to a H\"older class.
The definition of the H\"older class
$\Sigma(\beta,R)$ with $\beta \geq 1$ the smoothness parameter, and $R>0$, is given as follows:

\begin{equation*}
    \Sigma(\beta,R):= \left\{f \in \mathcal{C}^{\lfloor \beta \rfloor +1}(\R), ~ \left|f^{(\ell)}(x) - f^{(\ell)}(y)\right| \leq R|x-y|^{\lfloor \beta \rfloor - \ell}, ~ x,y \in \R\right\}.
\end{equation*}
In Section~\ref{subsec:generalRate}, we establish a general rate of convergence of the ERM-type classifier $\widehat{g}$. Then, in Section~\ref{subsec:fasterRate}, we strengthen the assumptions on our model by considering Assumption~\ref{ass:Margin}. Leveraging Assumption~\ref{ass:Margin},
we show that our classification model satisfies a margin type assumption (see {\it e.g.} \cite{audibert2007fast}). In this case, for $K=2$, we prove that the classifier $\widehat{g}$ achieves a faster rate of convergence.  

\subsection{General rates of convergence}
%%%%%%%%%%%%%%%%%%%%%%%%%%%%%%%%%%%%%%%%%%%%%%%%%%%%%%%%%%%%%%%%%%%%%%%%%%%%%%
%%%%%%%%%%%%%%%%%%%%%%%%%%%%%%%%%%%%%%%%%%%%%%%%%%%%%%%%%%%%%%%%%%%%%%%%%%%%%%
%%%%%%%%%%%%%%%%%%%%%%%%%%%%%%%%%%%%%%%%%%%%%%%%%%%%%%%%%%%%%%%%%%%%%%%%%%%%%%
\label{subsec:generalRate}

Let $\beta_1, \beta_2 \geq 1, ~ R_1,R_2 > 0$. We assume that
the functions $(b_k^*)_{k \in [K]}$ belong to the H\"older class $\Sigma(\beta_1,R_1)$ and $\sigma^{*2}$ belongs to the H\"older space $\Sigma(\beta_2,R_2)$. From Theorem~\ref{thm:EstimErrorERM} we deduce the following result.

\begin{coro}\label{coro:RateMA}
    Assume that $D_1:= D_N^1 \propto N^{1/(2\beta_1+1)}$,  $D_2:=D^2_N \propto N^{1/(2\beta_2+1)}$, and $\Delta_n = \mathrm{O}(1/N)$. 
    Let us denote $\beta_{\min} = \min(\beta_1,\beta_2)$ and $\beta_{\max} = \max(\beta_1, \beta_2)$. Under Assumptions~\ref{ass:RegEll} and \ref{ass:Novikov}, the following holds:
    \begin{equation*}
        \mathbb{E}\left[R(\widehat{g}) - R(g^*)\right] \leq C N^{-\beta_{\min}/(2\beta_{\min}+1)}\log^{\beta_{\max}+2}(N),
    \end{equation*}
    where $C$ is a constant that only depends on $K$, $R_1$, $R_2$, $\beta_1$, and $\beta_2$.
\end{coro}

Hereafter, we discuss our result regarding the existing literature. First, up to some extra factors, we obtain the same rate of convergence as in~\cite{denis2024nonparametric} but under more general assumptions. Specifically, in~\cite{denis2024nonparametric} rates of convergence are obtained in the case where the diffusion coefficient is known and constant. Furthermore, in the case where the diffusion coefficient is unknown, they show that for $\beta = 1$, the plug-in classifier achieves a rate of convergence of order $\exp(C\sqrt{\log(N)})N^{-1/5}$, which is clearly worst than the one obtained in Corollary~\ref{coro:RateMA}. Interestingly, it advocates the use of ERM estimators for prediction task than plug-in methods that are less flexible and require stronger assumptions.  
Second, the obtained rate of convergence in Corollary~\ref{coro:RateMA} is the same, up to a logarithmic factor, as the minimax rate of convergence for classification problem where the features are vectors of $\mathbb{R}^d$ (see {\it e.g.} \cite{bartlett2006convexity}). Finally, a similar rate of convergence of order $N^{-\beta/(2\beta+1)}$ is established in \cite{denis2020consistent} for an ERM type classifier, solution of the non-convex problem defined by Equation~\eqref{eq:eqERM01}, in the case where the diffusion coefficient of the diffusion model is assumed to be known and constant. In comparison the rates of convergence obtained in Corollary~\ref{coro:RateMA}, is obtained in the case where the diffusion coefficient is unknown. Furthermore, since our procedure relies on convex surrogate of the non-convex problem~\eqref{eq:eqERM01}, it is also dedicated for practical purpose.

In the next section, we study a faster rate of convergence of the excess risk of the ERM type classifier under stronger assumptions on our diffusion model.

\subsection{Faster rate of convergence}
%%%%%%%%%%%%%%%%%%%%%%%%%%%%%%%%%%%%%%%%%%%%%%%%%%%%%%%%%%%%%%%%%%%%%%%%%%%%%%
%%%%%%%%%%%%%%%%%%%%%%%%%%%%%%%%%%%%%%%%%%%%%%%%%%%%%%%%%%%%%%%%%%%%%%%%%%%%%%
%%%%%%%%%%%%%%%%%%%%%%%%%%%%%%%%%%%%%%%%%%%%%%%%%%%%%%%%%%%%%%%%%%%%%%%%%%%%%%
\label{subsec:fasterRate}

In this section, we study the case where the classifier $\widehat{g}$ achieves faster rates of convergence. 
For the sake of simplicity, we consider the binary classification framework ({\it e.g.} $K=2$). The case of the multiclass classification with $K \geq 3$ is more technical. The obtained result relies on a result provided in \cite{bartlett2006convexity}, which established a link between the excess risk of the ERM type classifier $\w{g}$ and the one of its corresponding score function $\w{\bf h}$ under  margin condition.

To establish a faster rate of convergence, we require more structural assumptions. In particular, we assume that $\sigma^*$ is known,  and we set, without loss of generality, $\sigma^* \equiv 1$.

\paragraph{Bayes rule.}

The label $Y$ is assumed to belong to $\{0,1\}$. In this case, the Bayes rule can be expressed as follows
\begin{equation*}
g^*(X) =  \one_{\{\pi^{*}_{1}(X) \geq 1/2\}} = \one_{\{\pi^{*}_{1}(X) \geq \pi^*_{0}(X)\}}.
\end{equation*}

Under the specific assumptions considered in this section,
we have that 
\begin{equation*}
\pi_1^*(X) = \dfrac{p_1^*q_1(X)}{p_0^* q_0(X) + p_1^* q_1(X)}, \;\; {\rm with} \;\; q_k(X) =  \exp\left(\int_{0}^{1}{b^{*}_{k}(X_s)dX_s} - \dfrac{1}{2}\int_{0}^{1}{b^{*2}_{k}(X_s)ds}\right).
\end{equation*}

Therefore, it is not difficult to see that
\begin{equation*}
g^*(X) = \one_{\{\int_{0}^{1}{(b^{*}_{1} - b^{*}_{0}(X_s))dX_s} \geq \int_{0}^{1}{(b^{*2}_{1} - b^{*2}_{0})(X_s)ds} + \log(p_0^*/p_1^*)\}}.
\end{equation*} 

\paragraph*{Margin assumption. }

We now present the assumption under which a margin assumption is fulfilled.
Specifically, we consider the following assumption on the drift functions $b_0^*$ and $b^*_1$.

\begin{assumption}
    \label{ass:Margin}
    The drift functions $b^{*}_{0}$ and $b^{*}_{1}$ are bounded on the real line, and the random variable 
    $$ Z := \frac{1}{\Delta_{b^{*}}}\int_{0}^{1}{(b^{*}_{1} - b^{*}_{0})(X_s)dW_s}, ~~ \mathrm{with} ~~ \Delta^{2}_{b^{*}} = \E\left[\int_{0}^{1}{(b^{*}_{1} - b^{*}_{0})^{2}(X_s)ds}\right] > 0, $$
    has a density function that is bounded on $\R$.
\end{assumption}
The existence of a density function of the random variable $Z$ may be proven using Maliavin calculus tools. More precisely, it can be obtained under the H\"ormander's condition (see \cite{nualart2006malliavin}, Chapter~$2$, Theorem~$2.3.3$). The constant ${\Delta}_{b^*}$ is the minimum separation distance between the drift functions $b^{*}_{0}$ and $b^{*}_{1}$ and then characterizes the margin between the two classes. In particular, the condition ${\Delta}_{b^*} > 0$ implies that the two classes do not overlap too much. 
Then, under Assumption~\ref{ass:Margin}, we establish the following result.

\begin{prop}
    \label{prop:TheMargin}
    Under Assumptions~\ref{ass:RegEll} and \ref{ass:Margin} and for all $\varepsilon \in (0,1/8)$, there exists a constant $C>0$ such that
\begin{align*}
    \P\left(0 < \left|\pi^{*}_{1}(X) - \frac{1}{2}\right| \leq \varepsilon\right) \leq &~ C\frac{12}{\Delta_{b^{*}}}\varepsilon.
\end{align*}
\end{prop}
Applying Theorem~$3$ in~\cite{bartlett2006convexity}, we deduce from the above proposition that
\begin{equation}\label{eq:LinkERmargin}
\mathbb{E}\left[R(\widehat{g})-R(g^*)\right] \leq \left(\mathbb{E}\left[R_2(\widehat{{\bf h}})-R_2({\bf h}^*)\right]\right)^{2/3}.   
\end{equation}

Therefore, we finally deduce from Theorem~\ref{thm:EstimErrorERM}, that Assumption~\ref{ass:Margin} implies 
the following result.

\begin{theo}\label{thm:RateMargin}
    Suppose that $b^{*}_{0}, b^{*}_{1} \in \Sigma(\beta,R)$, $D_N \propto N^{1/(2\beta+1)}$, and $\Delta_n = \mathrm{O}(1/N)$. Under Assumptions~\ref{ass:RegEll}, \ref{ass:Novikov} and \ref{ass:Margin}, the following holds:
    \begin{equation*}
        \mathbb{E}\left[R(\widehat{g}) - R(g^*)\right] \leq  C\left(N^{-4\beta/3(2\beta+1)}\log^{(4\beta+8)/3}(N)\right),
    \end{equation*}
    where $C > 0$ does not depend on $N$.
\end{theo}

Notably, Theorem~\ref{thm:RateMargin} shows that, compared to the result of Theorem~\ref{coro:RateMA},  classifier $\widehat{g}$ achieves a faster rate of convergence under the margin assumption that characterizes the behavior of the regression function $\pi_1^*$ around the decision boundary. In particular for $\beta \geq 3/2$, the classifier $\hat{g}$ achieves fast rates of convergence ({\it e.g.} faster than $N^{-1/2}$) up to some logarithmic factor.

Note that the obtained rate is based on Equation~\eqref{eq:LinkERmargin} which link the excess risk of the classifier $\w{g}$ to its corresponding score function $\w{\bf h}$. Hence, by considering the non-convex problem~\eqref{eq:eqERM01}, we may obtain faster rate of convergence~\citep{audibert2007fast}.

\section{Adaptive E.R.M.-type classifier for diffusion paths}
%%%%%%%%%%%%%%%%%%%%%%%%%%%%%%%%%%%%%%%%%%%%%%%%%%%%%%%%%%%%%%%%%%%%%%%%%%%%%%
%%%%%%%%%%%%%%%%%%%%%%%%%%%%%%%%%%%%%%%%%%%%%%%%%%%%%%%%%%%%%%%%%%%%%%%%%%%%%%
%%%%%%%%%%%%%%%%%%%%%%%%%%%%%%%%%%%%%%%%%%%%%%%%%%%%%%%%%%%%%%%%%%%%%%%%%%%%%%
\label{sec:AdaptiveERM}

In this section, we study an adaptive version of the ERM classifier defined in Section~\ref{sec:ERM-procedure}. Indeed, the rates of convergence provided in Section~\ref{subsec:generalRate} depend on the choice of the dimension of the space of approximation that relies on unknown regularity of the drift and diffusion functions. 
Specifically, we provide a data-driven procedure for the choice of the dimensions $D^1_N$ and $D^2_N$. In particular,
we show that the resulting classifier achieves same rates of convergence as its non-adaptive counterpart.

\paragraph*{Adaptive score function. }

The construction of the adaptive classifier is as follows. First, we build $\widehat{\bf h} = \widehat{\bf h}_{\w{D}_1, \w{D}_2}$. To this end, for all $D_1, D_2 \in \Xi_N \subset [N]$, we consider the collection of score functions $\left(\widehat{\bf h}_{D_1, D_2}\right)_{D_1, D_2 \in \Xi_N}$ with
\begin{equation*}
\widehat{\bf h}_{D_1, D_2} \in \argmin{{\bf h} \in \w{\bf H}_{D_1,D_2}} \widehat{R}_2({\bf h}).    
\end{equation*}

Then, our adaptive estimator $\widehat{\bf h} = \widehat{\bf h}_{\w{D}_1, \w{D}_2}$ is defined with
\begin{equation*}
(\w{D}_1, \w{D}_2) = \underset{D_1, D_2 \in \Xi_N}{\arg\min}{~\left\{\widehat{R}_2(\widehat{\bf h}_{D_1, D_2}) + \mathrm{pen}(D_1, D_2)\right\}},    
\end{equation*}

where $(D_1, D_2) \mapsto \mathrm{pen}(D_1, D_2)$ is the penalty function given by
\begin{equation}
\label{eq:penaltyFunction}
    \mathrm{pen}(D_1, D_2) := \kappa\dfrac{(D_1 + D_2)\log(N)}{N}, ~~ D_1, D_2 \in \Xi_N,
\end{equation}
and $\kappa > 0$ is a numerical constant.

The adaptive version of the ERM classifier is given by $\w{g} = g_{\w{\bf h}}$. Hereafter,  we establish  the following result, that provide a bound on the excess risk of the adaptive score function $\w{\bf h}$.

\begin{theo}
\label{thm:AdaptiveERMclassifier}
    Assume that  $\Xi_N = \{1, \ldots, N\}$. Under Assumption~\ref{ass:RegEll} and \ref{ass:Novikov}, there exists a constant $C>0$ such that
    \begin{equation*}
        \E\left[R_2(\w{\bf h}) - R_2({\bf h}^{*})\right] \leq 2\underset{D_N^1, D_N^2 \in \Xi_N}{\inf}{\left\{\underset{{\bf h} \in {\bf H}_{D_1, D_2}}{\inf}{R_2({\bf h}) - R_2({\bf h}^*)} + \mathrm{pen}\left(D_N^1, D_N^2\right)\right\}} + \dfrac{C}{N}.
   \end{equation*}
\end{theo}

As expected, Theorem~\ref{thm:AdaptiveERMclassifier} shows that the adaptive estimator achieves the same rates of convergence as the non adaptive one. Indeed, from Theorem~\ref{thm:EstimErrorERM}, up to some logarithmic factor we have that
for each $D_1, D_2 \in \Xi_N$
\begin{equation*}
\underset{{\bf h} \in {\bf H}_{D_1, D_2}}{\inf}{R_2({\bf h}) - R_2({\bf h}^*)} \leq C\left(\dfrac{1}{D^2_1} + \dfrac{1}{D^2_2}\right).
\end{equation*}

In view of the penalty function defined in Equation~\eqref{eq:penaltyFunction}, from the above inequality, we deduce the following.
If for each $k \in [K]$, $b^{*}_{k} \in \Sigma(\beta_1, R_1)$, and 
$\sigma^{*2} \in \Sigma(\beta_2, R_2)$, since
for each $\beta \geq 1$, $N^{1/(2\beta+1)} \in \Xi_N$, we have up to some logarithmic factors

%Then, since for each $\beta \geq 1$, $N^{1/(2\beta+1)} \in \Xi_N$, we deduce that if  
%for each $k \in [K]$, $b^{*}_{k} \in \Sigma(\beta_1, R_1)$, and 
%$\sigma^{*2} \in \Sigma(\beta_2, R_2)$, we have in view of the penalty function 

%that for $D_1 = N^{1/(2\beta_1+1)}$ and $D_2 = %N^{1/(2\beta_2+1)}$,
\begin{equation*}
%\label{eq:Rate-Adaptive-Score}
    \E\left[R_2(\w{\bf h}) - R({\bf h}^{*})\right] \leq C N^{-2\beta_{\min}/(2\beta_{\min}+1)}.
\end{equation*}
Therefore, we obtain that the excess risk of the adaptive estimator $\widehat{g} = g_{\widehat{\bf h}}$ is of order $N^{-\beta_{\min}/(2\beta_{\min}+1)}$, up to some logarithmic factor, more precisely:
\begin{equation*}
    \E\left[R(\w{g}) - R(g^{*})\right] \leq C N^{-\beta_{\min}/(2\beta_{\min}+1)}.
\end{equation*}
Hence, the adaptive estimator $\widehat{g}$ achieves same rates of convergence as the non adaptive one.

\section{Numerical study}
%%%%%%%%%%%%%%%%%%%%%%%%%%%%%%%%%%%%%%%%%%%%%%%%%%%%%%%%%%%%%%%%%%%%%%%%%%%%%%
%%%%%%%%%%%%%%%%%%%%%%%%%%%%%%%%%%%%%%%%%%%%%%%%%%%%%%%%%%%%%%%%%%%%%%%%%%%%%%
%%%%%%%%%%%%%%%%%%%%%%%%%%%%%%%%%%%%%%%%%%%%%%%%%%%%%%%%%%%%%%%%%%%%%%%%%%%%%%
\label{sec:NumericStudy}

This section is devoted to the evaluation of the numerical performance of the proposed ERM procedure on synthetic data. The simulation scheme is described in Section~\ref{subsec:models-simulations}. Then we provide details on the implementation of the procedure in Section~\ref{subsec:numerical-implementation}.
Finally, classification performance over the simulated data of the ERM predictor as well as a benchmark of competitors is presented in Section~\ref{subsec:numerical-results}.

\subsection{Experimental setup}
%%%%%%%%%%%%%%%%%%%%%%%%%%%%%%%%%%%%%%%%%%%%%%%%%%%%%%%%%%%%%%%%%%%%%%%%%%%%%
\label{subsec:models-simulations}
We study the numerical performance of the ERM classifier from simulated data.
We consider three different models with $K=3$ or $K=6$ classes for which Assumptions~\ref{ass:RegEll} and \ref{ass:Novikov} are satisfied. The horizon time is fixed to $T=1$ and the starting point to $X_0 = 0$. The description of the models is presented in Table~\ref{tab:DiffModels}. The data are simulated using the \texttt{Python} library \texttt{sdepy}. A visual description of the considered models
is displayed in Figure~\ref{fig:models}.

\begin{table}[!ht]
\small
\begin{center}
\renewcommand{\arraystretch}{1.25}
\begin{tabular}{l|c|c}
    \hline 
     \multirow{4}{*}{Model 1} & $b^{*}_{1}(x)$ & $-(x-1)$ \\ 
    
     & $b^{*}_{2}(x)$ & $-(x+1)$ \\ 
    
     & $b^{*}_{3}(x)$ & $-x$ \\ 
    
     & $\sigma^{*}(x)$ & $1$ \\
    \hline
    \multirow{4}{*}{Model 2} & $b^{*}_{1}(x)$ & $1/4+(3/4)\cos x$ \\ 
    
     & $b^{*}_{2}(x)$ & $-2\exp(-x^2)+\sin x$ \\ 
    
     & $b^{*}_{3}(x)$ & $4/\pi(x^2+1)$ \\ 
    
     & $\sigma^{*}(x)$ & $1$ \\
     \hline
     \multirow{4}{*}{Model 3} & $b^{*}_{1}(x)$ & $-1/2-(3/2)\cos^2(x)$\\ 
     & $b^{*}_{2}(x)$ & $1/4 + (3/4)\cos^2(x)$ \\ 
     & $b^{*}_{3}(x)$ & $1/2 + (3/2)\cos^2(x)$ \\ 
     & $b^{*}_{4}(x)$ & $(x-1)$\\
     & $b_5^*(x)$ & $-(x-1)$\\
     & $b_6^{*}(x)$ & $-(x-4)$ \\
     & $\sigma^{*}(x)$ & $1/10 + 9/(10\sqrt{1+x^2})$ \\
     \hline
\end{tabular}
\caption{{\small \textit{Description of the considered models. For each model, we provide the expression of the drift functions associated to each class, and the expression of the diffusion coefficient.  }}}
\label{tab:DiffModels}
\end{center}
\end{table}

\begin{figure}[!ht]
\begin{tabular}{c c c}
\includegraphics[scale=0.3]{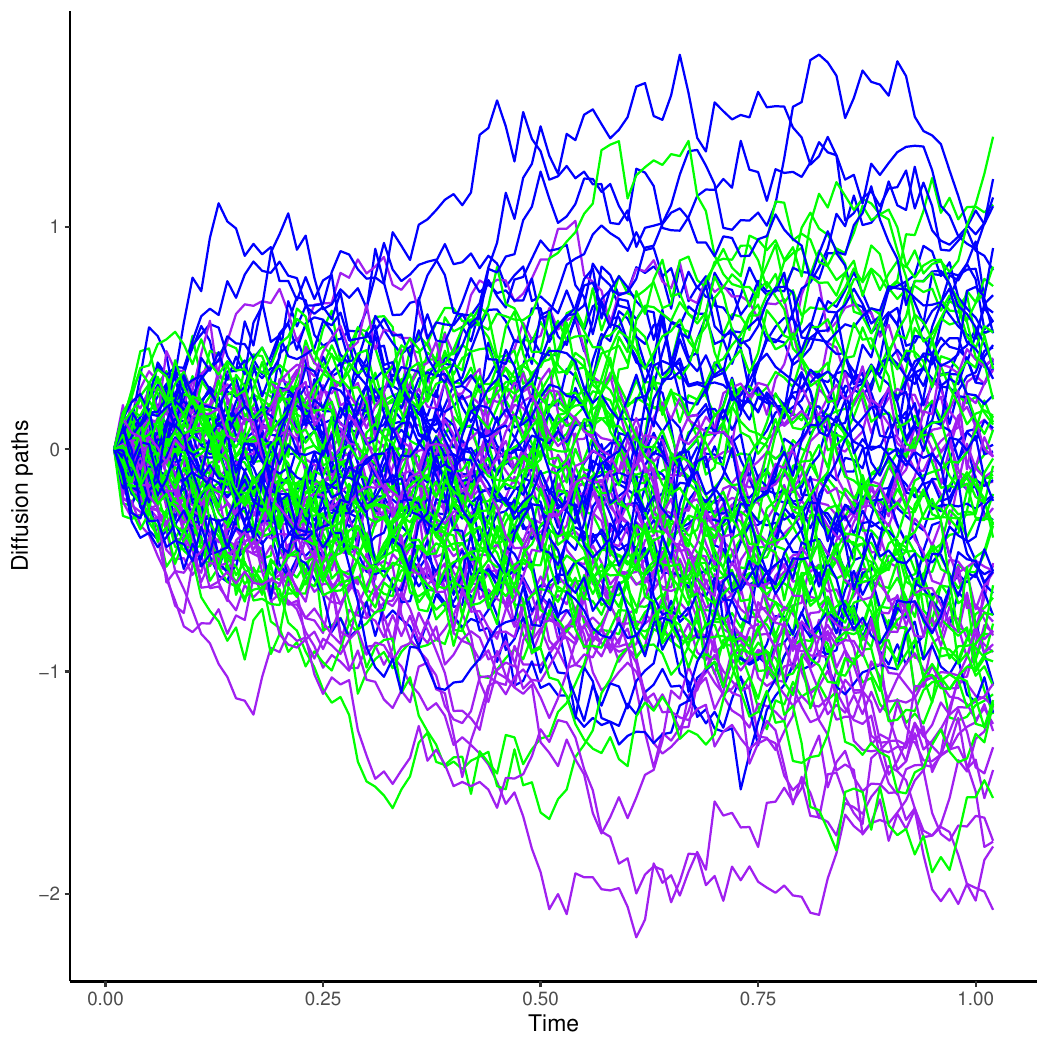} &
\includegraphics[scale=0.3]{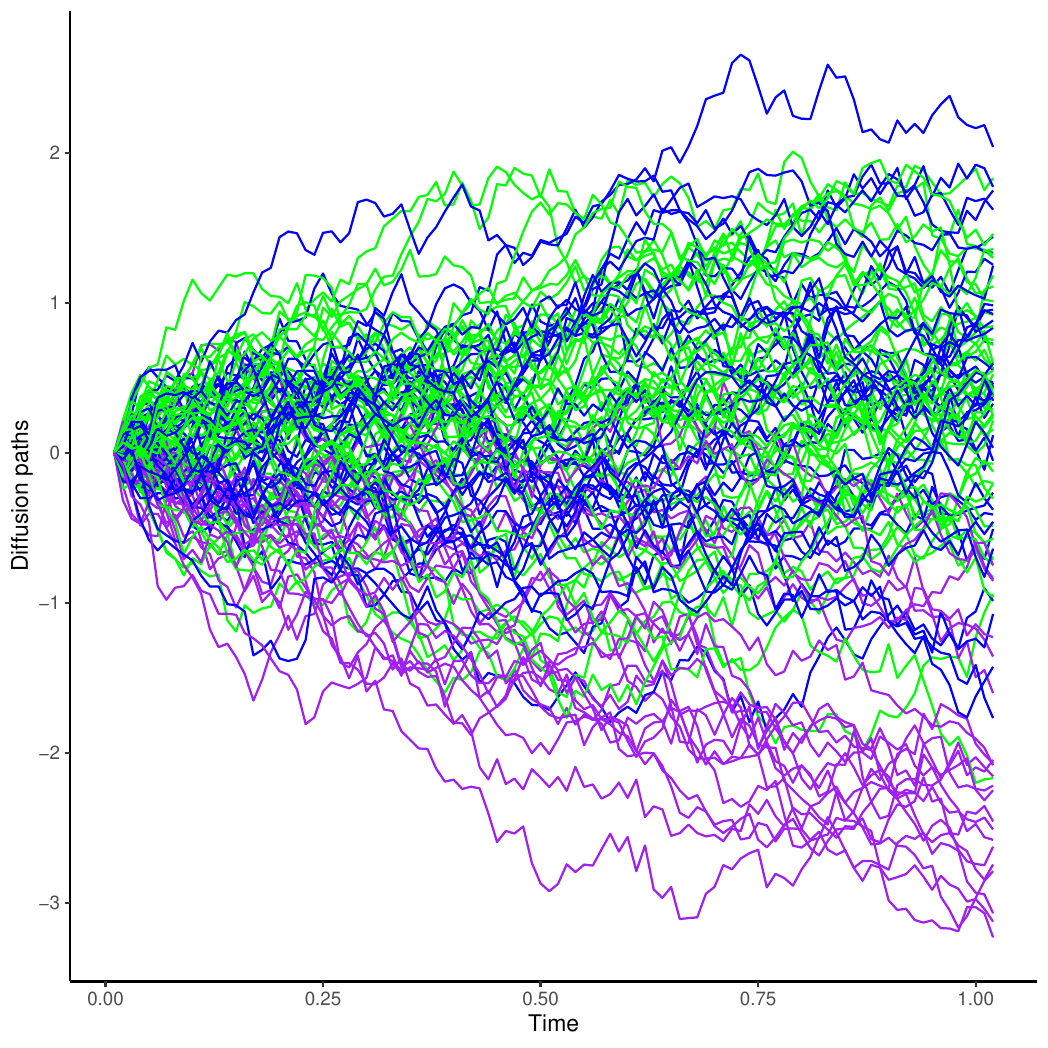} &
\includegraphics[scale=0.3]{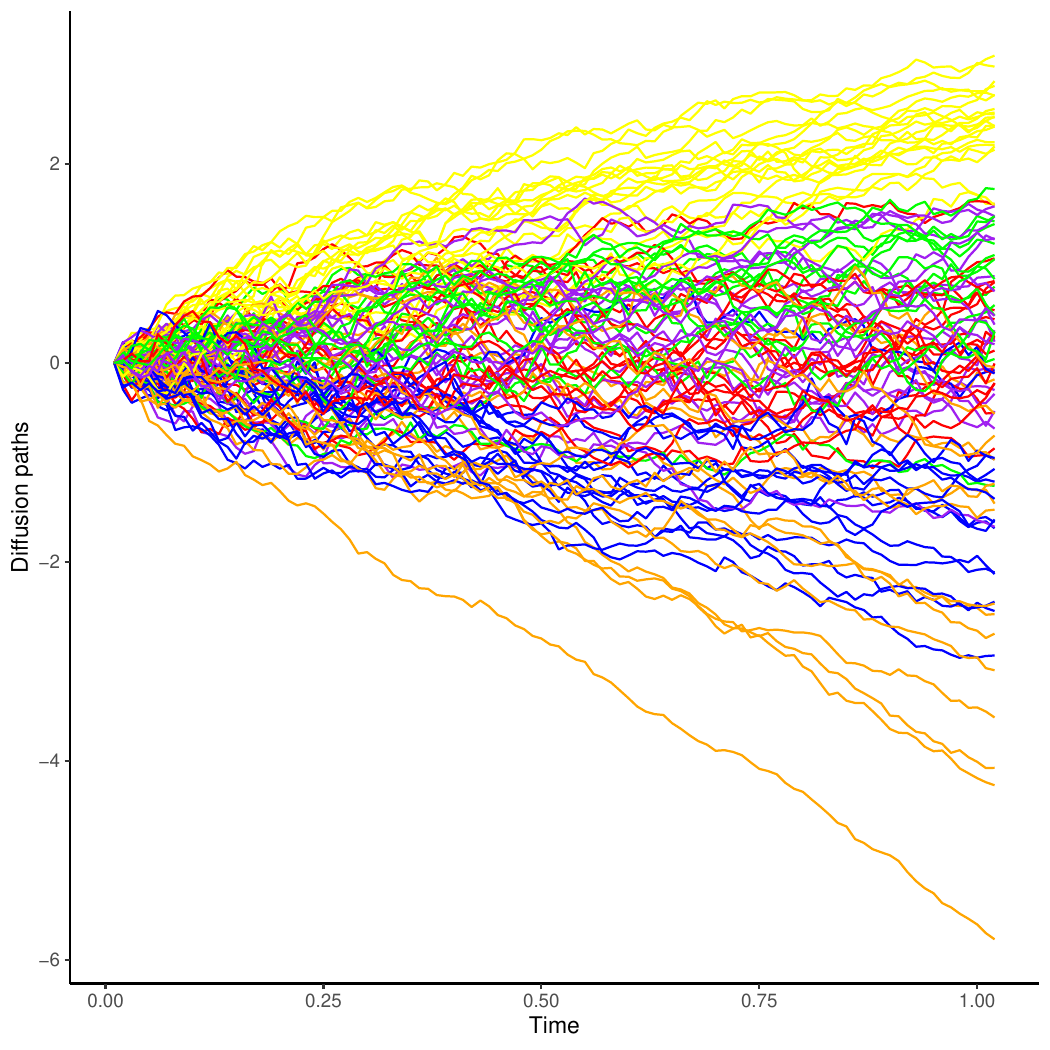}
\end{tabular}
\caption{Visual description of the considered models. Model~1 on the left, Model~2 in the middle, and Model~3 on the right.}
\label{fig:models}
\end{figure}

For Model~1 and Model~2, we consider a constant diffusion coefficient.
Model~1 is a mixture of Ornstein-Uhlenbeck processes. In particular, the drift functions associated to the classes are linear. Model~2 is composed of non linear functions and the estimation task may be more challenging compared to Model~1.
Finally, in Model~3, the diffusion coefficient is non constant and for this model, drift functions are both linear and non linear. The number of classes is $K=6$. Hence, with Model~3, we could observe the impact of the number of classes on the performance of the proposed procedure. Note that the diffusion coefficient is non constant and takes its values in the compact interval $[1/10, 1]$, which implies a less significant impact on the dispersion of the diffusion paths compared to the previous models.

To evaluate the performance of our ERM procedure, we compare its classification performance with the Bayes classifier and a benchmark of classification algorithms.

\paragraph*{Performance of the Bayes classifier. }

We evaluate the performance of the Bayes classifier according to the following scheme. We repeat $100$ times the following steps

\begin{enumerate}
    \item Simulate a learning sample $\mathcal{D}_{N}$ of size $N = 4000$ with $n = 500$
    \item For each diffusion path $\bar{X}$ of the learning sample $\mathcal{D}_{N}$,
    \begin{enumerate}
        \item compute for each class $k \in [K]$ the conditional probability $\bar{\pi}^{k}_{\mathbf{b}^*, \sigma^{*2}, \mathbf{p}^*}(\bar{X})$ as given in Section~\ref{subsec:score-functions} 
        \item Choose the class for which the conditional probability is the highest
    \end{enumerate}
    \item From the vector of predicted labels and the vector of true labels, compute the error rate of the classifier 
\end{enumerate}

The error rate is then deduced considering the average misclassification risk from the $100$ repetitions. 

\begin{table}[hbtp]
\centering
\renewcommand{\arraystretch}{1.5}
\begin{tabular}{c|c|c|c}
\hline
 & Model $1$ & Model $2$ & Model $3$ \\
\hline
 $\w{\mathcal{R}}(g^{*})$ & 0.41 \ (0.01) & 0.37 \ (0.01) & 0.31 \ (0.01)   \\
\hline
\end{tabular}
\caption{{\small \textit{Evaluation of the error rates of $g^{*}$.}}}
\label{tab:Risk-Bayes-ERM}
\end{table}

From Table~\ref{tab:Risk-Bayes-ERM}, we can see that regarding the classification task, Model~1 appears as the more challenging while the classification problem induced by Model~3 is easier. Although, the structure of the drift functions are more simple, the classes much overlap for Model~1 than for the others.

\paragraph*{Benchmark. }

 We compare the performance of the proposed ERM classifier with the plug-in classifier proposed in~\cite{denis2024nonparametric} which is a natural competitor. Indeed, the plug-in procedure provided in~\cite{denis2024nonparametric} is also consistent  for our classification problem.
 In our benchmark, we also consider the Depth-based classifier proposed for example in~\cite{lopez2006depth} or~\cite{mosler2017fast}. In particular, this algorithm is dedicated to the classification of functional data.
 Finally, we consider $K$-NN classifier, and  Random Forests. Not that this two algorithms are not tailored to handle functional data. However, there are known to achieve good performance for classification problems.

%\end{table}

%which are learning samples of size $N \in \{100, 1000\}$ composed of discrete time observations of the diffusion processes with time step $\Delta_n = 1/n$ with $n \in \{100, 250\}$. 

%The diffusion paths coming from models $1$, $2$ and $3$ are respectively displayed in Figures~\ref{fig:path-model1}, \ref{fig:path-model2} and \ref{fig:path-model3}. One can observe that Model 3 with $5$ classes appears to be a more complex one compared to the two other models. this complexity may be linked to the nature of some drift functions in the model, as well as the number of classes which makes it more difficult to distinguish one class from another. 

\subsection{Implementation of the ERM-type classifier}
\label{subsec:numerical-implementation}

In this section we provide a short description of the ERM classifier.
First, for the computation of the $B$-spline functions, we use the Python library \texttt{scikit-fda}~\cite{ramos-carreno++_2024_scikit-fda}.

For dimension parameters $D_1, D_2$, the sets $\mathcal{S}_{D_1}$, and $\widetilde{\mathcal{S}}_{D_2}$ are parameterized by the coefficient of the $B$-spline functions. Hence, the minimization problem
\begin{equation*}
%\label{eq:Minmization-RE}
\left(\widehat{{\bf b}}, \widehat{\sigma}^2\right) \in \argmin{{\bf b} \in \mathcal{S}_{D_1}^K, \sigma^2 \in \tilde{\mathcal{S}_{D_2}}} \widehat{R}_2\left(h_{{\bf b}, \sigma^2, \mathbf{{\hat{p}}}}\right)
\end{equation*}

can be solved by using a gradient descent algorithm. To this end, we use the function \texttt{minimize} with method \texttt{L-BFGS-B}. 

Then, we compute the ERM classifier 
$\widehat{g}  =g_{\widehat{\mathbf{h}}_{\widehat{D}_1, \widehat{D}_2}}$ 
defined as
\begin{equation*}
%\label{eq:Selection-Model}
\widehat{D}_1, \widehat{D}_2 = \underset{(D_1,D_2) \in \Xi}{\arg\min}{~\left\{\widehat{R}_2(\widehat{h}_{D_1,D_2})) + \kappa\dfrac{(D1+D_2+M)\log(N)}{N}\right\}}.    
\end{equation*}

We choose $\Xi = \{2, 4, 8\}$ and $\kappa =1$.

\subsection{Numerical results over simulated data}
\label{subsec:numerical-results}

The performance of the ERM predictor and the classifiers chosen as a benchmark
is assessed as follows. We chose $N \in \{100,1000\}$ and $n = 100$.
For each model presented in Section~\ref{subsec:models-simulations}, each $N \in \{100,1000\}$, and $M = 1000$, we repeat $30$ times the following steps.

\begin{enumerate}
    \item[(i)]Simulate a learning sample $\mathcal{D}_N$ and a test sample $\mathcal{D}_M$;
    \item[(ii)] Based on $D_N$, we compute all the considered predictors
    \item[(iii)] Based on $D_M$ we compute the error rates of the predictors.
\end{enumerate}
Then we compute the average error rates, and the standard deviation over the 
$30$ repetitions. The obtained results are presented in
Table~\ref{tab:ERM_100_100} and Table~\ref{tab:ERM_1000_100}.
Several comments can be made from this results.
First, the ERM classifier achieves the best performance for each model. In particular, it is slightly better than the plug-in classifier. Second, the error rates of the ERM classifier is closed the one of the Bayes classifier (see Table~\ref{tab:Risk-Bayes-ERM}, especially when $N=1000$. Third, we observe that the convergence of each classification rule appears to be faster for Model~1 and Model~3 than for Model~2. Finally, for Model~1, we can see that the ERM classifier achieves the same performance as the Bayes classifier for $N=1000$. It illustrates that the Model~1 is easier in term of estimation. 

\begin{comment}
\paragraph{Case of diffusion models with known diffusion coefficient $\sigma$.}

We propose to assess the performance of ERM-type classifiers when the diffusion coefficeint $\sigma$ of the diffusion model is assumed to be known. The statistical goal of this numerical study is to see the potential impact of such an assumption on the performance of the classifier compared to the results obtained in the general case.

The results of Tables \ref{tab:CompModel100} and \ref{tab:CompModel1000} show that the ERM-type classifier performs better for general diffusion models, where the drift and diffusion coefficient are assumed to be unkown. 

\end{comment}

\begin{table}[hbtp]
\centering
\renewcommand{\arraystretch}{1.5}
\begin{tabular}{c|c|c|c}
\hline
 Classifiers & Model $1$ & Model $2$ & Model $3$  \\
\hline
 ERM-type classifier & $\mathbf{0.43 ~ (0.02)}$ & $\mathbf{0.38 ~ (0.01)}$ & $\mathbf{0.34 ~ (0.02)}$ \\
 Plug-in type classifier & 0.45 \ (0.03) & 0.39 \ (0.01) & 0.37 \ (0.03) \\
 KNN based classifier & 0.48 \ (0.02) & 0.40 \ (0.01) & 0.39 \ (0.02) \\
 Random Forest based classifier & 0.45 \ (0.01) & 0.39 \ (0.01) & 0.37 \ (0.02) \\
 Depth-based classifier & 0.44 \ (0.01) & 0.40 \ (0.01) & 0.43 \ (0.02) \\
\hline
\end{tabular}
\caption{{\small \textit{Classification errors of the ERM-type classifier  from learning samples of size $N=100$ with $n=100$, and comparison with other classification methods}}}
\label{tab:ERM_100_100}
\end{table}

\begin{table}[hbtp]
\centering
\renewcommand{\arraystretch}{1.5}
\begin{tabular}{c|c|c|c}
\hline
 Classifiers & Model $1$ & Model $2$ & Model $3$  \\
\hline
 ERM-type classifier & $\mathbf{0.41 ~ (0.01)}$ & $\mathbf{0.38 ~ (0.01)}$ & $\mathbf{0.32 ~ (0.01)}$ \\
 Plug-in type classifier & 0.42 \ (0.01) & 0.38 \ (0.01) & 0.33 \ (0.01) \\
 KNN based classifier & 0.47 \ (0.01) & 0.40 \ (0.01) & 0.37 \ (0.01) \\
 Random Forest based classifier & 0.44 \ (0.01) & 0.39 \ (0.01) & 0.34 \ (0.01) \\
 Depth-based classifier & 0.43 \ (0.01) & 0.40 \ (0.01) & 0.41 \ (0.01) \\
\hline
\end{tabular}
\caption{{\small \textit{Classification errors of the ERM-type classifier  from learning samples of size $N=1000$ with $n=100$, and comparison with other classification methods}}}
\label{tab:ERM_1000_100}
\end{table}

\section{Discussion}
%%%%%%%%%%%%%%%%%%%%%%%%%%%%%%%%%%%%%%%%%%%%%%%%%%%%%%%%%%%%%%%%%%%%%%%%%%%%%
%%%%%%%%%%%%%%%%%%%%%%%%%%%%%%%%%%%%%%%%%%%%%%%%%%%%%%%%%%%%%%%%%%%%%%%%%%%%%
%%%%%%%%%%%%%%%%%%%%%%%%%%%%%%%%%%%%%%%%%%%%%%%%%%%%%%%%%%%%%%%%%%%%%%%%%%%%%
\label{sec:Conclusion}

In the present work, we propose a classification algorithm based on the minimization of the empirical $L_2$ risk. We consider the case where both drift functions $b^{*}_{k}, ~ k \in [K]$ and the diffusion coefficient $\sigma^{*}$ are unknown. Our procedure relies on the modeling of both drift and diffusion functions by $B$-spline functions.
We study the case where the drift functions belongs to a H\"older space with smoothness parameter $\beta_1 \geq 1$ and the function $\sigma^*$ belongs also to a H\"older space with smoothness parameter $\beta_2 \geq 1$
We derive, under mild assumptions, a rate of convergence of the resulting E.R.M. classifier of order $N^{-\beta_{\min}/(2\beta_{\min}+1)}$ with $\beta_{\min} = \min(\beta_1,\beta_2)$. In addition, in the case where $\sigma^*=1$, leveraging a margin assumption, we show that fast rates can be achieved. 
Finally, we propose an adaptive version of our procedure and establish its theoretical properties. A simulation study illustrates the performance of the proposed algorithm.

The results provided in the paper may be extended in several directions. A natural question for further research is to study the case of time in-homogeneous diffusion processes which covers a broader range of applications.
Finally, in several applications, the data may be modeled as multivariate diffusion processes. The extension of our procedure to the high-dimensional setting is an important challenge.

\newpage

\begin{center}

\Large{\bf Appendix}
\end{center}
\vspace*{0.25cm}

\appendix

%%%%%%%%%%%%%%%%%%%%%%%%%%%%%%%%%%%%%%%%%%%%%%%%%%%%%%%%%%%%%%%%%%%%%%%%%%%%%%
%%%%%%%%%%%%%%%%%%%%%%%%%%%%%%%%%%%%%%%%%%%%%%%%%%%%%%%%%%%%%%%%%%%%%%%%%%%%%%
%%%%%%%%%%%%%%%%%%%%%%%%%%%%%%%%%%%%%%%%%%%%%%%%%%%%%%%%%%%%%%%%%%%%%%%%%%%%%%
%\label{sec:Proofs}

In this section, we prove the main results of Sections~\ref{sec:ERM-procedure}, \ref{sec:RatesConv} and \ref{sec:AdaptiveERM}. To simplify our notations, we denote the time-step $\Delta_n = 1/n$ by $\Delta$, and the constants by $C>0$ and $c>0$, which can vary from a line to another. We can also use the notation $C_{\alpha} >0$ in case we want to precise the dependence of the constant $C>0$ on a parameter $\alpha$. Moreover, we denote the Bayes classifier $g_{h^{*}}$ by $g^{*}$, and any classifier $g_{h}$ by $g$, where $h \in {\bf H}$ is a score function.

\section{Technical results}

In this section, we gather some technical results that are used for the proof
of our results.

\begin{prop}[\cite{denis2024nonparametric}]
\label{prop:Bias}
Let $b$ be a $L$-Lipschitz function. there exists $\widetilde{b} \in \mathcal{S}_D$ such that for each $x \in [-\log(N),\log(N)]$
\begin{equation*}
\left|b(x)- \widetilde{b}(x)\right| \leq C \dfrac{\log(N)}{D}, \;\; \left|\widetilde{b}(x)\right|\leq C \log(N).    
\end{equation*}

Furthermore, if there exists $0 < b_{\min}, b_{\max}$ such that 
\begin{equation*}
 b_{\min} \leq b(x) \leq b_{\max},
\end{equation*}

we also have that $b_{\min} \leq \widetilde{b}(x) \leq b_{\max}$ for $x \in [-\log(N), \log(N)]$.
\end{prop}

\begin{lemme}[\cite{denis2024nonparametric}]
\label{lem:lemTech}
Let $b,\widetilde{b}, \sigma^2$, $\widetilde{\sigma}^2$ be four real valued functions, such that
for each $x \in \mathbb{R}$,
\begin{equation*}
\sigma^2(x) \geq \sigma_0^2 , \;\; {\rm and} \;\; \widetilde{\sigma}^2(x) \geq \sigma_0^2.    
\end{equation*}
The following holds.
\begin{equation*}
\left|\dfrac{b(x)}{\sigma^2(x)} -  \dfrac{\widetilde{b}(x)}{\widetilde{\sigma}^2(x)}\right| 
\leq \sigma_0^{-4}\left(\left|b(x) -\widetilde{b}(x)\right|\left|\widetilde{\sigma}^2(x)\right|+ \left|\sigma^2(x)-\widetilde{\sigma}^2(x)\right|\left|\widetilde{b}(x)\right|\right)
\end{equation*}
\end{lemme}

\begin{prop}[\cite{denis2024nonparametric}]
\label{prop:exitProba}
Under Assumption~\ref{ass:RegEll}, and~\ref{ass:Novikov},
there exist constants $C_1 > 0$ and $C_2 > 0$ such that
\begin{equation*}
\sup_{t \in [0,1]} \mathbb{P}\left(\left|X_t\right| \geq A\right) \leq \dfrac{C_1}{A} \exp\left(-C_2 A^2\right) 
\end{equation*}
\end{prop}

\begin{prop}[\cite{massart2007concentration}, Chap 2, Prop 2.8, Eq.(2.16), p23-24]
\label{prop:bernstein}
    Let $N \in \mathbb{N}\setminus\{0\}$ and $X_1, \ldots, X_N$ be $N$ independent copies of a square integrable random variable $X$ such that $|X| \leq L$ where $L \in (0, +\infty)$. For all $t > 0$, the following holds:
    \begin{equation*}
        \mathbb{P}\left(\mathbb{E}[X] - \dfrac{1}{N}\sum_{i=1}^{N}{X_i} > t\right) \leq \exp\left(-\dfrac{\dfrac{1}{2}N^2t^2}{\sum_{i=1}^{N}{\mathbb{E}(X_i^2)} + \dfrac{1}{3}LNt}\right).
    \end{equation*}
\end{prop}

\begin{prop}
\label{prop:EpsNet}
Let $\varepsilon > 0$ and $D_{\max} = \max(D_1,D_2)$. We have
\begin{equation*}
\log\left(\mathcal{N}\left(\varepsilon, \mathcal{S}_{D_1}^K \times \widetilde{\mathcal{S}}_{D_2}, \|\cdot\|_{\infty}\right) \right)\leq C_K D_{\max}\log\left(\dfrac{\sqrt{D_{\max}+M} \log^3(N)}{\varepsilon}\right).   
\end{equation*}
\end{prop}

\begin{proof}
First, for $D > 0$, we observe that for functions $b, \widetilde{b} \in \mathcal{S}_D$, 
we have for all $x \in \mathbb{R}$, 
\begin{equation*}
\left|b(x)- \widetilde{b}(x)\right| \leq \sum_{l=-M}^{D-1} |a_l-\tilde{a}_l|B_l(x).
\end{equation*}
Then from the Cauchy-Schwarz inequality, since $\sum_{l=-M}^{D-1}B_l^2(x) \leq 1$,
we deduce that
\begin{equation*}
\|b- \widetilde{b}\|_{\infty} \leq \left\|{\bf a}- {\bf \tilde{a}}\right\|_2.   
\end{equation*}
Therefore, since from \cite{lorentz1996constructive}, {\it Chap 15, Prop 1.3, p487}, we have
\begin{equation*}
\mathcal{N}\left(\varepsilon,\bar{B}\left(0,R\right), \|\cdot\|_2\right)
\leq \left(\dfrac{3R}{\varepsilon}\right)^D,    
\end{equation*}

where $\mathcal{N}\left(\varepsilon,\bar{B}\left(0,R\right), \|\cdot\|_2\right)$ is the minimum number of open balls of radious $\varepsilon$ that completely cover the closed ball $\bar{B}(0,R) \subset \mathbb{R}^D$, 
we deduce that for each $D > 0$, and from the definition of $\mathcal{S}_D$ that
\begin{equation}
\label{eq:eqEpsNet}
\log\left(\mathcal{N}\left(\varepsilon, \mathcal{S}_D, \|\cdot\|_{\infty}\right)\right) \leq  CD\log\left(\dfrac{\sqrt{D+M} \log^3(N)}{\varepsilon}\right).   
\end{equation}

Now let $\mathcal{S}_{D,\varepsilon}$ be an $\varepsilon$-net of $\mathcal{S}_D$.
We observe that 
\begin{equation*}
\left\{x \mapsto \sigma^2(x)  = \widetilde{\sigma}^2(x) \one_{\{\widetilde{\sigma}^2(x)\geq 1/\log(N)\}}   + \frac{1}{\log(N)} \one_{\{\widetilde{\sigma}^2(x)\leq 1/\log(N)\}}, \;\;\widetilde{\sigma}^2 \in \mathcal{S}_{D,\varepsilon}\right\},
\end{equation*}
is an $\varepsilon$-net of $\widetilde{S}_D$. Therefore with Equation~\eqref{eq:eqEpsNet}, we get the desired result.
\end{proof}

\begin{coro}\label{coro:CardH}
    Let $\varepsilon > 0$, and denote by $\w{\bf H}_{D_1, D_2}^{\varepsilon}$ the $\varepsilon-$net of the space $\w{\bf H}_{D_1, D_2}$ with respect to the supremum norm given by
    \begin{equation*}
        \w{\bf H}_{D_1, D_2}^{\varepsilon} := \left\{{\bf h}_{\bb_{\varepsilon}, \sigma_{\varepsilon}^2, \w{\bf p}}, ~~ \bb_{\varepsilon} \in \mathcal{S}_{D_1}^{K, \varepsilon}, ~~ \sigma_{\varepsilon}^2 \in \widetilde{\mathcal{S}}_{D_2}^{\varepsilon}\right\},
    \end{equation*}
    where $\mathcal{S}_{D_1}^{K, \varepsilon}$ and $\widetilde{\mathcal{S}}_{D_2}^{\varepsilon}$ are the respective $\varepsilon-$nets of the spaces $\mathcal{S}_{D_1}^K$ and $\widetilde{\mathcal{S}}_{D_2}$. The following hold:
    \begin{itemize}
        \item ${\rm Card}\left(\w{\bf H}_{D_1, D_2}^{\varepsilon}\right) = \mathcal{N}\left(\varepsilon, \mathcal{S}_{D_1}^K \times \widetilde{\mathcal{S}}_{D_2}, \|.\|_{\infty}\right)$
        \item For all $\h \in \w{\bf H}_{D_1, D_2}$ and for all $\h_{\varepsilon} \in \w{\bf H}_{D_1, D_2}^{\varepsilon}$, there exists a constant $C>0$ such that 
        \begin{equation*}
            \E\left[R_2(\h) - R_2(\h_{\varepsilon})\right] \leq CD_1\log^3(N)\varepsilon^2.
        \end{equation*}
        \item For all $\h \in \w{\bf H}_{D_1, D_2}$ and for all $\h_{\varepsilon} \in \w{\bf H}_{D_1, D_2}^{\varepsilon}$, there exists a constant $C>0$ such that
        \begin{equation*}
            \E\left[\w{R}_2(\h) - \w{R}_2(\h_{\varepsilon})\right] \leq C\left(D_1^{1/2}n^{1/2}(\log(N))^{3/2} + D_1\log^3(N)\right)\varepsilon.
        \end{equation*}
    \end{itemize}
\end{coro}

\begin{proof}
    The first result of Corollary~\ref{coro:CardH} is a direct consequence of Proposition~\ref{prop:EpsNet}. 
    
    \paragraph{Second result.}
    
    The score functions $\h \in \w{\bf H}_{D_1, D_2}$ and $\h_{\varepsilon} \in \w{\bf H}_{D_1, D_2}^{\varepsilon}$ are given for all $k \in [K]$ by
    \begin{equation*}
        h^k(\bar{X}) = 2\bar{\pi}_{\bb, \sigma^2, \w{\bf p}}^k(\bar{X}) - 1, ~~ {\rm and} ~~ h_{\varepsilon}^k(\bar{X}) = 2\bar{\pi}_{\bb_{\varepsilon}, \sigma_{\varepsilon}^2, \w{\bf p}}^k(\bar{X}) - 1,
    \end{equation*}
    
    where for each $k \in [K]$,
    \begin{equation*}
        \bar{\pi}_{\bb, \sigma^2, \w{\bf p}}^k(\bar{X}) = {\rm softmax}_k^{\w{\bf p}}(\bar{X}), ~~ {\rm and} ~~ \left\|\partial_{x_j} {\rm softmax}_k^{\w{\bf p}}\right\|_{\infty} \leq 1, ~~ j \in [K].
    \end{equation*}
    
    Then we have
\begin{multline*}
    R_2({\bf h})- R_2({\bf h}_{\varepsilon}) \leq  4\sum_{k=1}^{K}{\mathbb{E}\left[\left|\bar{\pi}^k_{{\bf b}, \sigma^2, \widehat{\bf p}}(\bar{X}) - \bar{\pi}^{k}_{{\bf b}_{\varepsilon}, \sigma_{\varepsilon}^2, \widehat{\bf p}}(\bar{X})\right|^2\right]}\\
    \leq C\sum_{k=1}^{K}{\mathbb{E}\left|\bar{F}^{k}_{{\bf b}, \sigma^2}(\bar{X}) - \bar{F}^{k}_{{\bf b}_{\varepsilon}, \sigma_{\varepsilon}^2}(\bar{X})\right|^2}\\
    \leq C\left[\sum_{k=1}^{K}{\left\|b_k - b_{k,\varepsilon}\right\|^{2}_{n}}+\sum_{k=1}^{K}{\left\|b_{k}\left(\sigma^{2}-\sigma^{2}_{\varepsilon}\right)\right\|^{2}_{n}}\right].
\end{multline*}

So we obtain 
\begin{equation*}
    \mathbb{E}\left[ R_2({\bf h})- R_2({\bf h}_{\varepsilon})\right] \leq C\left[\sum_{k=1}^{K}{\left\|b_k - b_{k, \varepsilon}\right\|_{\infty}^2} + D_1\log^3(N)\left\|\sigma^2 - \sigma_{\varepsilon}^2\right\|_{\infty}^2\right] \leq CD_1\log^{3}(N)\varepsilon^{2}.
\end{equation*}

\paragraph{Third result.}

For all $\h \in \w{\bf H}_{D_1, D_2}$ and $\h_{\varepsilon} \in \w{\bf H}_{D_1, D_2}^{\varepsilon}$, we have
\begin{equation}  
\label{eq:ExcessEmpRisk}
\widehat{R}_2({\bf h}_{\varepsilon}) - \widehat{R}_2({\bf h}) = \frac{1}{N}\sum_{j=1}^{N}{\sum_{k=1}^{K}{\left\{\left(1 - Z^{j}_{k}h^{k}_{\varepsilon}(\bar{X}^{j})\right)^2 - \left(1 - Z^{j}_{k}h^{k}(\bar{X}^{j})\right)^2\right\}}},
\end{equation}
where for all $j \in [N]$ and $k \in [K]$, one has $Z^{j}_{k} \in \{-1,1\}$ and,
\begin{multline*}    
    \left|\left(1 - Z^{j}_{k}h^{k}_{\varepsilon}(\bar{X}^{j})\right)^2 - \left(1 - Z^{j}_{k}h^{k}(\bar{X}^{j})\right)^2\right| = \left|h^{k} - h^{k}_{\varepsilon}\right|(\bar{X}^{j}) \times \left|2-Z^{j}_{k}(h^{k}(\bar{X}^{(j)}) + h^{k}_{\varepsilon}(\bar{X}^{j}))\right|\\
    \leq 4\left|h^{k} - h^{k}_{\varepsilon}\right|(\bar{X}^{j}).
\end{multline*}
Then, we obtain:
\begin{multline*}
    \left|\widehat{R}_2({\bf h}_{\varepsilon}) - \widehat{R}_2({\bf h})\right| \leq \frac{1}{N}\sum_{j=1}^{N}{\sum_{k=1}^{K}{\left|\left(1 - Z^{j}_{k}h^{k}_{\varepsilon}(\bar{X}^{j})\right)^2 - \left(1 - Z^{j}_{k}h^{k}(\bar{X}^{j})\right)^2\right|}}\\
    \leq \frac{8}{N}\sum_{j=1}^{N}{\sum_{k=1}^{K}{\left|\bar{\pi}^k_{{\bf b}, \sigma^2, \widehat{\bf p}}(\bar{X}^{j}) - \bar{\pi}^{k}_{{\bf b}_{\varepsilon}, \sigma_{\varepsilon}^2, \widehat{\bf p}}(\bar{X}^{j})\right|}}
\end{multline*}
and,
\begin{multline}
\label{eq:excess-phi-emp-risk}
    \mathbb{E}\left[\left|\widehat{R}_2({\bf h}_{\varepsilon}) - \widehat{R}_2({\bf h})\right|\right] \leq \frac{8}{N}\sum_{j=1}^{N}{\sum_{k=1}^{K}{\mathbb{E}\left[\left|\bar{\pi}^k_{{\bf b}, \sigma^2, \widehat{\bf p}}(\bar{X}^{j}) - \bar{\pi}^{k}_{{\bf b}_{\varepsilon}, \sigma_{\varepsilon}^2, \widehat{\bf p}}(\bar{X}^{j})\right|\right]}}\\
    \leq \dfrac{C}{N}\sum_{j=1}^{N}{\sum_{k=1}^{K}{\mathbb{E}\left[\left|\bar{F}^{k}_{{\bf b}, \sigma^2}(\bar{X}^j)-\bar{F}^{k}_{\widehat{\bf b}_{\varepsilon}, \sigma_{\varepsilon}^2}(\bar{X}^j)\right|\right]}}.
\end{multline}

For all $j \in [N]$ and $k \in [K]$, one has:
\begin{align*}
\bar{F}^{k}_{{\bf b}, \sigma^2}(\bar{X}^{j}) := \sum_{i=0}^{n-1}{\left(\frac{b_k}{\sigma^{2}}(X^{j}_{i\Delta})(X^{j}_{(i+1)\Delta} - X^{j}_{i\Delta})-\frac{\Delta}{2}\frac{b^{2}_{k}}{\sigma^{2}}(X^{j}_{i\Delta})\right)}
\end{align*}

and, using the Cauchy Schwarz inequality, we have
\begin{multline*}
    \left|\bar{F}^{k}_{{\bf b}, \sigma^2}(\bar{X}^{j})  - \bar{F}^{k, \varepsilon}_{{\bf b}, \sigma^2}(\bar{X}^{j}) \right| \leq \left|\sum_{i=0}^{n-1}{\left(\frac{b_k}{\sigma^{2}} - \frac{b_{k,\varepsilon}}{\sigma^{2}_{\varepsilon}}\right)(X^{j}_{i\Delta})(X^{j}_{(i+1)\Delta}-X^{j}_{i\Delta})}\right|\\
    + \frac{\Delta}{2}\left|\sum_{i=0}^{n-1}{\left(\frac{b^{2}_{k}}{\sigma^{2}}-\frac{b^{2}_{k,\varepsilon}}{\sigma^{2}_{\varepsilon}}\right)}(X^{j}_{i\Delta})\right|\\
    \leq \left(\sum_{i=0}^{n-1}{\left|\frac{b_k}{\sigma^{2}}-\frac{b_{k,\varepsilon}}{\sigma^{2}_{\varepsilon}}\right|^2(X^{j}_{i\Delta})}\right)^{1/2} \left(\sum_{i=0}^{n-1}{\left|X^{j}_{(i+1)\Delta}-X^{j}_{i\Delta}\right|^2}\right)^{1/2}\\
    +\frac{\Delta}{2}\sum_{i=0}^{n-1}{\left|\frac{b^{2}_{k}}{\sigma^{2}}-\frac{b^{2}_{k,\varepsilon}}{\sigma^{2}_{\varepsilon}}\right|(X^{j}_{i\Delta})}.
\end{multline*}
For all $k \in [K]$, we have:
\begin{multline*}
\left|\dfrac{b_k}{\sigma^{2}} - \dfrac{b_{k,\varepsilon}}{\sigma^{2}_{\varepsilon}}\right| \leq \dfrac{\left|b_k - b_{k,\varepsilon}\right|}{\sigma^{2}} + \dfrac{\left|b_{k,\varepsilon}\right|\times\left|\sigma^{2} - \sigma^{2}_{\varepsilon}\right|}{\sigma^{2}\sigma^{2}_{\varepsilon}}\\
\leq \sigma_0^{-2}\left\|b_k - b_{k,\varepsilon}\right\|_{\infty} + \sigma_0^{-4}(D_1+M)^{1/2}\log^{3/2}(N)\left\|\sigma^{2} - \sigma^{2}_{\varepsilon}\right\|_{\infty}\\
\leq CD_1^{1/2}\log^{3/2}(N)\varepsilon,
\end{multline*}
where $C>0$ is a constant, and
\begin{multline*}
    \left|\dfrac{b^{2}_k}{\sigma^{2}} - \dfrac{b^{2}_{k,\varepsilon}}{\sigma^{2}_{\varepsilon}}\right| \leq 2\sigma_0^{-2}(D_1+M)^{1/2}\log^{3/2}(N)\left|b_k - b_{k,\varepsilon}\right| + \sigma_0^{-4}(K_N+M)\log^{3}(N)\left|\sigma^{2} - \sigma^{2}_{\varepsilon}\right|\\
    \leq 4\sigma_0^{-2}D_1^{1/2}\log^{3/2}(N)\left\|b_k - b_{k,\varepsilon}\right\|_{\infty} + 2\sigma_0^{-4}D_1\log^{3}(N)\left\|\sigma^{2} - \sigma^{2}_{\varepsilon}\right\|_{\infty} \\
    \leq C\log^{3}(N)D_1\varepsilon.
\end{multline*}
It results that for all $k \in [K]$, we have:
\begin{equation}\label{eq:excess-phi-emp-risk2}
    \left|\bar{F}^{k}_{{\bf b}, \sigma^2}(\bar{X}^{j})  - \bar{F}^{k}_{{\bf b}_{\varepsilon}, \sigma_{\varepsilon}^2}(\bar{X}^{j}) \right| \leq  CD_1^{1/2}\log^{3/2}(N)\varepsilon\sqrt{n} \left(\sum_{i=0}^{n-1}{\left|X^{j}_{(i+1)\Delta}-X^{j}_{i\Delta}\right|^2}\right)^{1/2} + C\log^{3}(N)D_1\varepsilon,
\end{equation}

where $C > 0$ is a new constant. Besides, for all $k\in[\![0,n-1]\!]$,
\begin{multline*}
    \left|X^{1}_{(k+1)\Delta}-X^{1}_{k\Delta}\right|^2 \leq 2\left|\int_{k\Delta}^{(k+1)\Delta}{b^{*}_{Y}(X_s)ds}\right|^2 + \left|\int_{k\Delta}^{(k+1)\Delta}{\sigma^{*}(X_s)dW_s}\right|^2\\
    \leq C\Delta^2 \left(1 + \underset{t \in [0,1]}{\sup}{|X_t|}\right)^2 + 2\left|\int_{k\Delta}^{(k+1)\Delta}{\sigma^{*}(X_s)dW_s}\right|^2.
\end{multline*}

Then,
\begin{multline}\label{eq:Interm-result}
    \mathbb{E}\left[\sum_{k=0}^{n-1}{\left|X^{(1)}_{(k+1)\Delta}-X^{(1)}_{k\Delta}\right|^2}\right] \leq C\Delta\left(1 + \mathbb{E}\left[\underset{t \in [0,1]}{\sup}{|X_t|}\right]\right)+ 2\sum_{k=0}^{n-1}{\mathbb{E}\left(\int_{k\Delta}^{(k+1)\Delta}{\sigma^{*2}(X_s)ds}\right)}\\
    \leq C\left(1 + \mathbb{E}\left[\underset{t \in [0,1]}{\sup}{|X_t|}\right]\right) + 2c_1^2.
\end{multline}

We deduce from Equations~\eqref{eq:Interm-result}, \eqref{eq:excess-phi-emp-risk2} and \eqref{eq:excess-phi-emp-risk} that there exists a constant $C>0$ such that
\begin{equation*}
    \mathbb{E}\left[\widehat{R}_2({\bf h}_{\varepsilon}) - \widehat{R}_2({\bf h})\right] \leq C\left(D_1^{1/2}n^{1/2}(\log(N))^{3/2} + \log^{3}(N)D_1\right)\varepsilon.
\end{equation*}

\end{proof}

\section{Proofs of Section~\ref{sec:ERM-procedure}}
%%%%%%%%%%%%%%%%%%%%%%%%%%%%%%%%%%%%%%%%%%%%%%%%%%%%%%%%%%%%%%%%%%%%%%%%%%%%%%
%%%%%%%%%%%%%%%%%%%%%%%%%%%%%%%%%%%%%%%%%%%%%%%%%%%%%%%%%%%%%%%%%%%%%%%%%%%%%%

\subsection{Proof of Theorem~\ref{thm:EstimErrorERM}}

\begin{proof}

The proof of the theorem goes in several steps.
We recall that for functions ${\bf b}, \sigma^2$, and vector of probability distribution ${\bf p} = \left(p_1, \ldots,p_K\right)$,
the score function ${\bf h}_{{\bf b}, \sigma^2, {\bf p}}$ is defined for each $k \in [K]$ as 
\begin{equation*}
h_{{\bf b}, \sigma^2, {\bf p}}^k(\bar{X}) = h_{{\bf b}, \sigma^2, {\bf p}}^k(\bar{X}) = 2 \bar{\pi}^k_{{\bf b}, \sigma^2, {{\bf p}}}(\bar{X}) -1
\end{equation*}

We also introduce the score function $\widetilde{{\bf h}}$
defined as 
\begin{equation*}
\widetilde{{\bf h}} \in \argmin{{\bf h} \in \w{\bf H}_{D_1,D_2}} R_2({\bf h}).
\end{equation*}

Hence $\widetilde{{\bf h}}$ can be viewed as the oracle counterpart of $\widehat{{\bf h}}$.
Thanks to Proposition~\ref{prop:Bias}, we also introduce the functions $\bar{\bb} = \left(\bar{b}_1, \ldots\bar{b}_K\right) \in \mathcal{S}_{D_1}^K$, and  $\bar{\sigma}^2 \in \widetilde{\mathcal{S}}_{D_2}$ such that for $x \in [-\log(N),\log(N)]$, and $k \in [K]$
\begin{equation*}
\left|b^*_k(x)- \bar{b}_k(x)\right| \leq C \dfrac{\log(N)}{D_1}, \;\;   \left|\sigma^{*2}(x)- \bar{\sigma}^2(x)\right| \leq C \dfrac{\log(N)}{D_2}.  
\end{equation*}

Set
\begin{equation*}
    \mathcal{E}_{\bf h} = R_2({\bf h}) - R_2({\bf h}^{*}), ~ \w{\mathcal{E}}_{\bf h} = \w{R}_2({\bf h}) - \w{R}_2({\bf h}^{*}).
\end{equation*}

We have
\begin{equation*}
    R_2(\w{\bf h}) - R_2({\bf h}^*) = 2\w{\mathcal{E}}_{\w{\bf h}} + \mathcal{E}_{\w{\bf h}} - 2\w{\mathcal{E}}_{\w{\bf h}} \leq 2\underset{{\bf h} \in \w{\bf H}_{D_N^1, D_N^2}}{\inf}{\left\{\w{R}_2({\bf h}) - \w{R}_2({\bf h}^*)\right\}} + \mathcal{E}_{\w{\bf h}} - 2\w{\mathcal{E}}_{\w{\bf h}}.
\end{equation*}

Then we obtain,
\begin{multline*}
    \E\left[R_2(\w{\bf h}) - R_2({\bf h}^*) \biggm\vert \mathcal{Y}_N\right] \leq 2\E\left[\underset{{\bf h} \in \w{\bf H}_{D_N^1, D_N^2}}{\inf}{\left\{\w{R}_2({\bf h}) - \w{R}_2({\bf h}^*)\right\}} \biggm\vert \mathcal{Y}_N\right] + \E\left[\mathcal{E}_{\w{\bf h}} - 2\w{\mathcal{E}}_{\w{\bf h}} \biggm\vert \mathcal{Y}_N\right] \\
    \leq 2\underset{{\bf h} \in \w{\bf H}_{D_N^1, D_N^2}}{\inf}{\left\{\E\left[\w{R}_2({\bf h}) - \w{R}_2({\bf h}^*) \biggm\vert \mathcal{Y}_N\right]\right\}} + \E\left[\mathcal{E}_{\w{\bf h}} - 2\w{\mathcal{E}}_{\w{\bf h}} \biggm\vert \mathcal{Y}_N\right]\\
    = 2\underset{{\bf h} \in \w{\bf H}_{D_N^1, D_N^2}}{\inf}{\left\{R_2({\bf h}) - R_2({\bf h}^*) \right\}} + \E\left[\mathcal{E}_{\w{\bf h}} - 2\w{\mathcal{E}}_{\w{\bf h}} \biggm\vert \mathcal{Y}_N\right].
\end{multline*}

We deduce that
\begin{equation*}
    \E\left[R_2(\w{\bf h}) - R_2({\bf h}^*)\right] \leq 2\E\left[R_2(\widetilde{\bf h}) - R_2({\bf h}^*)\right] + \E\left[\mathcal{E}_{\w{\bf h}} - 2\w{\mathcal{E}}_{\w{\bf h}}\right].
\end{equation*}

We start with the following decomposition
\begin{equation*}
R_2(\widetilde{\bf h}) - R_2({\bf h}^*) = \left(R_2(\widetilde{{\bf h}}) - R_2(\h_{\bar{\bb}, \bar{\sigma}, \widehat{\p}})\right) + \left(R_2(\h_{\bar{\bb}, \bar{\sigma}, \widehat{\p}}) - R_2({\bf h}^*)\right).
\end{equation*}

From the definition of $\widetilde{\h}$, since $\h_{\bar{\bb}, \bar{\sigma}^2, \widehat{\p}} \in \w{\bf H}_{D_1,D_2}$, we deduce that
\begin{multline}\label{eq:eqDecomp1}
    \E\left[R_2(\w{\bf h}) - R_2({\bf h}^*)\right] \leq 2\E\left[R_2(\widetilde{\bf h}) - R_2({\bf h}^*)\right] + \E\left[\mathcal{E}_{\w{\bf h}} - 2\w{\mathcal{E}}_{\w{\bf h}}\right]\\
         \leq 2 \left(R_2(\h_{\bar{\bb}, \bar{\sigma}, \widehat{\p}}) - R_2({\bf h}^*)\right) + \E\left[\mathcal{E}_{\w{\bf h}} - 2\w{\mathcal{E}}_{\w{\bf h}}\right].
\end{multline}

Now, since for each $k \in [K]$, 
$h_k^*(X) = \mathbb{E}\left[Z_k|X\right]$, we have
\begin{equation*}
R_2(\h_{\bar{\bb}, \bar{\sigma}^2, \widehat{\p}}) - R_2({\bf h}^*) = \sum_{k=1}^K \mathbb{E}\left[\left(\h_{\bar{\bb}, \bar{\sigma}^2, \widehat{\p}}(X)- \h^*(X)\right)^2\right]  = 4 \sum_{k=1}^K \mathbb{E}\left[\left(\bar{\pi}^k_{\bar{\bb}, \bar{\sigma}^2, \widehat{\p}}(\bar{X}) -  \pi^*_k(X)\right)^2\right].
\end{equation*}

Therefore, from the above equation, we obtain that
\begin{multline}
\label{eq:eqDecomp2}
R_2(\h_{\bar{\bb}, \bar{\sigma}^2, \widehat{\p}}) - R_2({\bf h}^*) \leq 12\sum_{k=1}^K \mathbb{E}\left[\left(\bar{\pi}^k_{\bar{\bb}, \bar{\sigma}^2, \widehat{\p}}(\bar{X}) -  \bar{\pi}^k_{{\bb}^{*}, {\sigma}^{*2}, \widehat{\p}}(\bar{X})\right)^2\right] \\+ 12\sum_{k=1}^K \mathbb{E}\left[\left(\bar{\pi}^k_{{\bb}^*, {\sigma}^{*2}, \widehat{\p}}(\bar{X}) -  \bar{\pi}^k_{{\bb}^{*}, {\sigma}^{*2}, {\p}^*}(\bar{X})\right)^2\right] \\+12 \sum_{k=1}^K \mathbb{E}\left[\left(\bar{\pi}^k_{{\bb}^*, {\sigma}^{*2}, \p^*}(\bar{X}) - \pi^*_k(X) \right)^2\right].
\end{multline}

In the above decomposition, there are three terms. The first one is related to the bias term, the second one is related to estimation of the weights ${\bf p}^*$, and the last term is related to the discretization error. Finally, in view of Equation~\eqref{eq:eqDecomp1}, a last term is the variance term which is given
by the excess risk $E_{\widehat{\h}}$ where for each $\h \in \w{\bf H}_{D_1,D_2}$ 
\begin{equation*}
E_{\h} = R_2(\h)-R_2(\widetilde{\h}).
\end{equation*}

In the following steps, we provide a bound for each of them.

\paragraph*{Discretization term. }

From Lemma $7.4$ in~\cite{denis2020consistent}, we have that for each $k \in [K]$
\begin{equation}\label{eq:discretization-term}
\mathbb{E}\left[\left(\bar{\pi}^k_{{\bf b}^{*}, \sigma^{*2}, {\bf p}^*}(\bar{X}) -  \pi^*_k(X)\right)^2\right] \leq C_K \Delta,    
\end{equation}

\paragraph*{Estimation of the weights. } In this step, we consider 
the  term
\begin{equation*}
\sum_{k=1}^K \mathbb{E}\left[\left(\bar{\pi}^k_{{\bb}^*, {\sigma}^{*2}, \widehat{\p}}(\bar{X}) -  \bar{\pi}^k_{{\bb}^{*}, {\sigma}^{*2}, {\p}^*}(\bar{X})\right)^2\right],   
\end{equation*}

which is related to the estimation error of the weights ${\bf p}^*$.

We recall that for each $k \in [K]$, 
\begin{equation*}
\widehat{p}_k = \dfrac{1}{N} \sum_{i=1}^N\tilde{Y}_i,
\end{equation*}
where the sample $\mathcal{Y}_N = \left\{\tilde{Y}_1, \ldots, \tilde{Y}_N\right\}$ is independent of $\mathcal{D}_N$.
We introduce the event
\begin{equation*}
\mathcal{A} = \left\{\min(\widehat{\bf p}) \geq \dfrac{{\bf p}^*_{\min}}{2}\right\}.
\end{equation*}

We have
\begin{multline}
\label{eq:eqDecomp3}
\sum_{k=1}^K \mathbb{E}\left[\left(\bar{\pi}^k_{{\bb}^*, {\sigma}^{*2}, \widehat{\p}}(\bar{X}) -  \bar{\pi}^k_{{\bb}^{*}, {\sigma}^{*2}, {\p}^*}(\bar{X})\right)^2\right]\\
= \sum_{k=1}^K \mathbb{E}\left[\left({\rm softmax}_k^{\widehat{{\bf p}}}(\bar{F}^k_{{\bf b}^{*2}, \sigma^{*2}}(\bar{X}))   - {\rm softmax}_k^{{\bf p}^*}(\bar{F}^k_{{\bf b}^*, \sigma^{*2}}(\bar{X})\right)^2\right].
\end{multline}

Now we observe that on the event $\mathcal{A}$, we have that
$\p \mapsto {\rm softmax}_k^{{{\bf p}}}(\bar{F}^k_{{\bf b}^*, \sigma^{*2}}(\bar{X}))$ satisfies
for each $j \in [K]$
\begin{equation*}
\left|\partial_j  {\rm softmax}_k^{{{\bf p}}}(\bar{F}^k_{{\bf b}^*, \sigma^{*2}}(\bar{X}))\right| \leq \dfrac{2}{\p_{\min}^*}.
\end{equation*}

Therefore, from the mean value theorem, we get
\begin{equation*}
\left|{\rm softmax}_k^{\widehat{{\bf p}}}(\bar{F}^k_{{\bf b}^*, \sigma^{*2}}(\bar{X}))   - {\rm softmax}_k^{{\bf p}^*}(\bar{F}^k_{{\bf b}^*, \sigma^{*2}}(\bar{X})\right| \leq C_K \left\|\widehat{\p}- \p^*\right\|_2.
\end{equation*}

Hence from Equation~\ref{eq:eqDecomp3}, since for each tuple $(\bb,\sigma^2, \p)$,  the functions $\bar{\pi}_{\bb,\sigma^2,\p}$ are bounded by $1$, we deduce that
\begin{equation*}
\sum_{k=1}^K \mathbb{E}\left[\left(\bar{\pi}^k_{{\bb}^*, {\sigma}^{*2}, \widehat{\p}}(\bar{X}) -  \bar{\pi}^k_{{\bb}^{*}, {\sigma}^{*2}, {\p}^*}(\bar{X})\right)^2\right] \leq C_K\left(\left\|\widehat{\p}- \p^*\right\|^2_2+ \mathbb{P}\left(\mathcal{A}^c\right)\right).   
\end{equation*}

Now from Hoeffding Inequality, we have that $\mathbb{P}\left(\mathcal{A}^c\right) \leq C\exp\left(-CN\p_{\min}^2\right) \leq \frac{1}{N}$ for $N$ large enough.
Furthermore, we have that  $\mathbb{E}\left[\left\|\widehat{\p}- \p^*\right\|^2_2\right] \leq \dfrac{1}{N} $. Hence, the above equation yields
\begin{equation}\label{eq:weights-bound}
\sum_{k=1}^K \mathbb{E}\left[\left(\bar{\pi}^k_{{\bb}^*, {\sigma}^{*2}, \widehat{\p}}(\bar{X}) -  \bar{\pi}^k_{{\bb}^{*}, {\sigma}^{*2}, {\p}^*}(\bar{X})\right)^2\right] \leq C_K \dfrac{1}{N}.
\end{equation}

\paragraph*{Bias term. }

In this step, we consider the term
\begin{equation*}
\sum_{k=1}^K \mathbb{E}\left[\left(\bar{\pi}^k_{\bar{\bb}, \bar{\sigma}^2, \widehat{\p}}(\bar{X}) -  \bar{\pi}^k_{{\bb}^{*}, {\sigma}^{*2}, \widehat{\p}}(\bar{X})\right)^2\right].    
\end{equation*}

Since $x \mapsto {\rm softmax}^{\widehat{\p}}(x)$ is $1$ Lipschitz, we deduce that
\begin{equation}\label{eq:Bias}
\sum_{k=1}^K \mathbb{E}\left[\left(\bar{\pi}^k_{\bar{\bb}, \bar{\sigma}^2, \widehat{\p}}(\bar{X}) -  \bar{\pi}^k_{{\bb}^{*}, {\sigma}^{*2}, \widehat{\p}}(\bar{X})\right)^2\right] \leq 4K^2\sum_{k=1}^K
\mathbb{E}\left[\left(\bar{F}_{\bar{\bb}, \bar{\sigma}^2}^k\left(\bar{X}\right) - \bar{F}_{{\bb}^{*}, {\sigma}^{*2}}^k(\bar{X})\right)^2\right].
\end{equation}

Then, for each $k \in [K]$, we have
%\begin{equation*}
%\mathbb{E}\left[\left(\bar{F}_{\bar{\bb}, \bar{\sigma}}\left(\bar{X}\right) - \bar{F}_{{\bb}^{*}, {\sigma}^*}(\bar{X})\right)^2\right] =  \mathbb{E}\left[\left(\sum_{j=0}^{n-1}{\dfrac{(\bar{b}_{k})}{\bar{\sigma}^{2}}\right)(X_{k\Delta})-(X_{(k+1)\Delta} -\dfrac{(\bar{b}_{k})}{\bar{\sigma}^{2}}(X_{k\Delta})\right) X_{k\Delta})} - \dfrac{\Delta}{2}\sum_{k=0}^{n-1}{\dfrac{b_{k}^2}{\sigma^2}}(X_{k\Delta})- \dfrac{b_{k}^2}{\sigma^2}}(X_{k\Delta})\right)^2\right]    
%\end{equation*}
\begin{multline*}
\mathbb{E}\left[\left(\bar{F}_{\bar{\bb}, \bar{\sigma}^2}^k\left(\bar{X}\right) - \bar{F}_{{\bb}^{*}, {\sigma}^{*2}}^k(\bar{X})\right)^2\right] = \\
\mathbb{E}\left[\left(
\sum_{j=1}^{n-1}\left(\dfrac{\bar{b}_{k}}{\bar{\sigma}^{2}}(X_{j\Delta}) -  
\dfrac{b^*_{k}}{\sigma^{*2}}(X_{j\Delta})\right)(X_{(j+1)\Delta}-X_{j\Delta}) - 
\dfrac{\Delta}{2}\sum_{j=1}^{n-1}\left(\dfrac{\bar{b}^2_k}{\bar{\sigma}^2}(X_{j\Delta})- \dfrac{b^{*2}_k}{\sigma^{*2}}(X_{j\Delta})\right)
\right)^2
\right].
\end{multline*}

Let us introduce $\xi(s)=j\Delta \ \mathrm{if} \ s\in[j\Delta,(j+1)\Delta[$. Then, the {\it r.h.s.} in 
the above equation can be expressed as 
\begin{equation*}
\mathbb{E}\left[\left(\int_{0}^1 \left(\dfrac{\bar{b}_{k}}{\bar{\sigma}^{2}}(X_{\xi(s)}) -  
\dfrac{b^*_{k}}{\sigma^{*2}}(X_{\xi(s)})\right){\rm d}X_{s} - \dfrac{1}{2}\int_{0}^1 \left(\dfrac{\bar{b}^2_k}{\bar{\sigma}^2}(X_{\xi(s)})- \dfrac{b^{*2}_k}{\sigma^{*2}}(X_{\xi(s)})\right){\rm d}s\right)^2\right].
\end{equation*}

Therefore, we deduce that
\begin{multline}
\label{eq:eqDecomp4}
\mathbb{E}\left[\left(\bar{F}_{\bar{\bb}, \bar{\sigma}^2}^k\left(\bar{X}\right) - \bar{F}_{{\bb}^{*}, {\sigma}^{*2}}^k(\bar{X})\right)^2\right]  \leq  3 \mathbb{E}\left[\left(  \int_{0}^1 \left(\dfrac{\bar{b}_{k}}{\bar{\sigma}^{2}}(X_{\xi(s)}) -  
\dfrac{b^*_{k}}{\sigma^{*2}}(X_{\xi(s)})\right) b^*_Y(X_s) {\rm d}{s} \right)^2\right] \\
                              + 3 \mathbb{E}\left[\left(  \int_{0}^1 \left(\dfrac{\bar{b}_{k}}{\bar{\sigma}^{2}}(X_{\xi(s)}) -  
\dfrac{b^*_{k}}{\sigma^{*2}}(X_{\xi(s)}) \right) {\rm \sigma^*(X_s)d}W_{s} \right)^2\right] \\
                             + \dfrac{3}{4} \mathbb{E}\left[\left(   \int_{0}^1 \left(\dfrac{\bar{b}^2_k}{\bar{\sigma}^2}(X_{\xi(s)})- \dfrac{b^{*2}_k}{\sigma^{*2}}(X_{\xi(s)})\right){\rm d}s \right)^2\right].
\end{multline}

Now we bound each of the three terms in the {\it r.h.s.} of the above equation.
First, using the Cauchy-Schwarz inequality, we have that
\begin{multline}
\label{eq:eqDecomp5}
\mathbb{E}\left[\left(  \int_{0}^1 \left(\dfrac{\bar{b}_{k}}{\bar{\sigma}^{2}}(X_{\xi(s)}) -  
\dfrac{b^*_{k}}{\sigma^{*2}}(X_{\xi(s)}) \right) b^*_Y(X_s) {\rm d}{s} \right)^2\right]\\
\leq 
\mathbb{E}\left[\int_{0}^1 \left(\dfrac{\bar{b}_{k}}{\bar{\sigma}^{2}}(X_{\xi(s)}) -  
\dfrac{b^*_{k}}{\sigma^{*2}}(X_{\xi(s)})\right)^2 b^{*2}_Y(X_s){\rm d}s\right].
\end{multline}

Then from Lemma~\ref{lem:lemTech}, we deduce that for each $s \in [0,1]$
\begin{multline}
\label{eq:eqUtile1}
\left(\dfrac{\bar{b}_{k}}{\bar{\sigma}^{2}}(X_{\xi(s)}) -  
\dfrac{b^*_{k}}{\sigma^{*2}}(X_{\xi(s)})\right)^2  \leq \\ \sigma_0^{-8}\left( \left|\bar{b}_k(X_{\xi(s)}) -b_k^*(X_{\xi(s)})\right|\left|\sigma^{*2}(X_{\xi(s)})\right|+ \left|\sigma^{*2}(X_{\xi(s)})-\bar{\sigma}^2(X_{\xi(s)})\right|\left|\bar{b}_k(X_{\xi(s)})\right|\right)^2
\end{multline}

Note that since $\bar{b}_k \in \mathcal{S}_{D_1}$, thanks to Proposition~\ref{prop:Bias}, we have that for $x \in [-\log(N), \log(N)]$,  $\left|\bar{b}_k(x)\right| \leq \log(N)$.
Besides for $x \notin [-\log(N),\log(N)]$, $\bar{b}_k(x) = 0$.
Therefore, since $\sigma^{*2}(x) \leq \sigma_1^2$, we deduce that for $X_{\xi(s)} \in [-\log(N),\log(N)]$,
we have from Proposition~\ref{prop:Bias}
\begin{equation}
\label{eq:eqUtile2}
\left(\dfrac{\bar{b}_{k}}{\bar{\sigma}^{2}}(X_{\xi(s)}) -  
\dfrac{b^*_{k}}{\sigma^{*2}}(X_{\xi(s)})\right)^2  \leq C \log^4(N) \left(\dfrac{1}{D^2_1}+\dfrac{1}{D^2_2}\right). 
\end{equation}

Furthermore, for $X_{\xi(s)} \notin [-\log(N), \log(N)]$, we observe that since $b_k^*$ is Lipschitz,
\begin{equation}
\label{eq:eqUtile3}
\left(\dfrac{\bar{b}_{k}}{\bar{\sigma}^{2}}(X_{\xi(s)}) -  
\dfrac{b^*_{k}}{\sigma^{*2}}(X_{\xi(s)})\right)^2  \leq   \sigma_0^{-8} \sigma_1^2 \left| b_k^*(X_{\xi(s)})\right|^2 \leq C \sup_{t \in [0,1]} \left|X_t\right|^2.
\end{equation}

Since $b_Y^*$ is also Lipschitz, we have for each $s \in [0,1]$,
\begin{equation*}
 \left|b^{*2}_Y(X_s)\right| \leq \sup_{t \in [0,1]} \left|X_t\right|.
\end{equation*}

Therefore, in view of Equations~\ref{eq:eqDecomp5},~\ref{eq:eqUtile1}, ~\ref{eq:eqUtile2},~\ref{eq:eqUtile3},
we get 
\begin{multline*}
\mathbb{E}\left[\left(\int_{0}^1 \left(\dfrac{\bar{b}_{k}}{\bar{\sigma}^{2}}(X_{\xi(s)}) -  \dfrac{b^*_{k}}{\sigma^{*2}}(X_{\xi(s)}) \right) b^*_Y(X_s) {\rm d}s \right)^2\right] \leq \\ C\log^4(N)\left(\dfrac{1}{D_1^2}+\dfrac{1}{D_2^2}\right)\mathbb{E}\left[\sup_{t \in [0,1]} \left|X_t\right|^2\int_{0}^1 \one_{X_{\xi(s)} \in [-\log(N),\log(N)]} {\rm d}s\right] \\ + 
C \mathbb{E}\left[\sup_{t \in [0,1]} \left|X_t\right|^2\int_{0}^1 \one_{X_{\xi(s)} \notin [-\log(N),\log(N)]} {\rm d}s\right].
\end{multline*}

Since $\mathbb{E}\left[\sup_{t \in [0,1]} \left|X_t\right|^p\right] < C$ for $p \geq 1$,  Proposition~\ref{prop:exitProba} and
the above equation leads to
\begin{multline}
\label{eq:eqDecomp6}
 \mathbb{E}\left[\left(  \int_{0}^1 \left(\dfrac{\bar{b}_{k}}{\bar{\sigma}^{2}}(X_{\xi(s)}) -  \dfrac{b^*_{k}}{\sigma^{*2}}(X_{\xi(s)}) b^*_Y(X_s) {\rm d}{s}\right)  \right)^2\right] \leq \\   C\left[\log^4(N)\left(\dfrac{1}{D_1}+\dfrac{1}{D_2}\right) + \sum_{j=0}^{n-1} \mathbb{P}\left(\left|X_{j\Delta}\right| \geq \log(N)\right)\right] \\
 \leq C\left[\log^4(N)\left(\dfrac{1}{D_1^2}+\dfrac{1}{D_2^2}\right) + \dfrac{n}{\log(N)} \exp\left(-c\log^2(N)\right)\right].
\end{multline}

Now, we study 
\begin{equation*}
\mathbb{E}\left[\left(  \int_{0}^1 \left(\dfrac{\bar{b}_{k}}{\bar{\sigma}^{2}}(X_{\xi(s)}) -  
\dfrac{b^*_{k}}{\sigma^{*2}}(X_{\xi(s)}) \right) \sigma^*(X_s){\rm d}W_{s}  \right)^2\right].
\end{equation*}

We observe that
\begin{equation*}
\mathbb{E}\left[\left(\int_{0}^1 \left(\dfrac{\bar{b}_{k}}{\bar{\sigma}^{2}}(X_{\xi(s)}) -  
\dfrac{b^*_{k}}{\sigma^{*2}}(X_{\xi(s)})\right) \sigma^*(X_s){\rm d}W_{s} \right)^2\right] =\mathbb{E}\left[\int_{0}^1 \left(\dfrac{\bar{b}_{k}}{\bar{\sigma}^{2}}(X_{\xi(s)}) -  
\dfrac{b^*_{k}}{\sigma^{*2}}(X_{\xi(s)})\right)^2 \sigma^{*2}(X_s){\rm d} {s}\right]  .
\end{equation*}

Using similar arguments as for the bound of Equation~\eqref{eq:eqDecomp6},
we also obtain that
\begin{multline}\label{eq:Bound-Term2}
\mathbb{E}\left[\left(\int_{0}^1 \left(\dfrac{\bar{b}_{k}}{\bar{\sigma}^{2}}(X_{\xi(s)}) -  
\dfrac{b^*_{k}}{\sigma^{*2}}(X_{\xi(s)})\right) \sigma^*(X_s){\rm d}W_{s} \right)^2\right]\\
\leq C\left[\log^4(N)\left(\dfrac{1}{D_1^2}+\dfrac{1}{D_2^2}\right) + \dfrac{n}{\log(N)} \exp\left(-c\log^2(N)\right)\right].
\end{multline}

Finally, it remains to bound the term
\begin{equation*}
    \mathbb{E}\left[\left(\int_{0}^1 \left(\dfrac{\bar{b}^2_k}{\bar{\sigma}^2}(X_{\xi(s)})- \dfrac{b^{*2}_k}{\sigma^{*2}}(X_{\xi(s)})\right){\rm d}s \right)^2\right].
\end{equation*}

We have,
\begin{multline*}
\mathbb{E}\left[\left(\int_{0}^1 \left(\dfrac{\bar{b}^2_k}{\bar{\sigma}^2}(X_{\xi(s)})- \dfrac{b^{*2}_k}{\sigma^{*2}}(X_{\xi(s)})\right){\rm d}s \right)^2\right] = \mathbb{E}\left[\left(\int_{0}^{1}{\dfrac{\bar{b}_k^2\sigma^{*2} - b_k^{*2}\bar{\sigma}^2}{\bar{\sigma}^2\sigma^{*2}}(X_{\xi(s)})ds}\right)^2\right]\\
= \mathbb{E}\left[\left(\int_{0}^{1}{\dfrac{\left(\bar{b}_k^2 - b_k^{*2}\right)\sigma^{*2} + b_k^{*2}\left(\sigma^{*2} - \bar{\sigma}^2\right)}{\bar{\sigma}^2\sigma^{*2}}(X_{\xi(s)})ds}\right)^2\right].
\end{multline*}

We then obtain from Proposition~\ref{prop:Bias} that
\begin{multline}\label{eq:Bound-Term3}
    \mathbb{E}\left[\left(\int_{0}^1 \left(\dfrac{\bar{b}^2_k}{\bar{\sigma}^2}(X_{\xi(s)})- \dfrac{b^{*2}_k}{\sigma^{*2}}(X_{\xi(s)})\right){\rm d}s \right)^2\right] \leq C\left(\log^{2}(N)\mathbb{E}\left[\left\|\bar{b}_k - b_k^*\right\|_{n,k}^2\right] + \log^{4}(N)\mathbb{E}\left[\left\|\bar{\sigma}^2 - \sigma^{*2}\right\|_n^2\right]\right)\\
    \leq C\left(\dfrac{\log^{4}(N)}{D_1^2} + \dfrac{\log^{6}(N)}{D_2^2}\right).
\end{multline}

We finally obtain from Equations~\eqref{eq:Bias}, \eqref{eq:eqDecomp4}, \eqref{eq:eqDecomp6}, \eqref{eq:Bound-Term2} and \eqref{eq:Bound-Term3} that
\begin{equation}\label{eq:Bound-Bias}
\sum_{k=1}^K \mathbb{E}\left[\left(\bar{\pi}^k_{\bar{\bb}, \bar{\sigma}, \widehat{\p}}(\bar{X}) -  \bar{\pi}^k_{{\bb}^{*}, {\sigma}^*, \widehat{\p}}(\bar{X})\right)^2\right] \leq C\left[\dfrac{\log^4(N)}{D_1^2} + \dfrac{\log^6(N)}{D_2^2} + \dfrac{n}{\log(N)}\exp\left(-c\log^2(N)\right)\right],  
\end{equation}

and from Equations~\eqref{eq:Bound-Bias}, \eqref{eq:weights-bound}, \eqref{eq:discretization-term} and \eqref{eq:eqDecomp2}, we obtain
\begin{multline}\label{eq:bayes-step-p}
    \E\left[R_2(\widetilde{\bf h}) - R_2({\bf h}^*)\right] \leq \E\left[R_2(\h_{\bar{\bb}, \bar{\sigma}, \widehat{\p}}) - R_2({\bf h}^*)\right] \\
    \leq C\left[\dfrac{\log^4(N)}{D_1^2} + \dfrac{\log^6(N)}{D_2^2} + \dfrac{n}{\log(N)}\exp\left(-c\log^2(N)\right) + \dfrac{1}{N} + \Delta\right]. 
\end{multline}

Going back to Equation~\eqref{eq:eqDecomp1}, we have
\begin{multline}
\label{eq:eqDecomp}
    R_2(\widehat{{\bf h}}) - R_2({\bf h}^*) \leq \E\left[\mathcal{E}_{\w{\bf h}} - 2\w{\mathcal{E}}_{\w{\bf h}}\right]\\
     + C\left[\dfrac{\log^4(N)}{D_1^2} + \dfrac{\log^6(N)}{D_2^2} + \dfrac{n}{\log(N)}\exp\left(-c\log^2(N)\right) + \dfrac{1}{N} + \Delta\right].
\end{multline}

We now focus on the estimation term which is on the $r.h.s$ of the above equation.

\paragraph*{Estimation term. }

Recall that $\widetilde{\bf h}$, which is viewed as the oracle counterpart of $\widehat{\bf h}$, is defined as follows:
\begin{equation*}
    \widetilde{{\bf h}} \in \argmin{{\bf h} \in \w{\bf H}_{D_1,D_2}} R_2({\bf h}).
\end{equation*}

For any score function ${\bf h} \in \w{\bf H}_{D_1,D_2}$, set
\begin{equation*}
    E_{\bf h}:= R_2({\bf h}) - R_2(\widetilde{\bf h}) ~~ \mathrm{and} ~~ \widehat{E}_{\bf h} := \widehat{R}_2({\bf h}) - \widehat{R}_2(\widetilde{\bf h}).
\end{equation*}

We have,
\begin{multline*}
    \mathcal{E}_{\w{\bf h}} - 2\w{\mathcal{E}}_{\w{\bf h}} = R_2(\w{\bf h}) - R_2({\bf h}^*) - 2\left(\w{R}_2(\w{\bf h}) - \w{R}_2({\bf h}^*)\right)\\
    = E_{\w{\bf h}} - 2\w{E}_{\w{\bf h}} + R_2(\widetilde{\bf h}) - R_2({\bf h}^*) - 2\left(\w{R}_2(\widetilde{\bf h}) - \w{R}_2({\bf h}^*)\right)\\
    \leq E_{\widehat{\bf h}} - E_{\widehat{\bf h}_{\varepsilon}} + 2\left(\widehat{E}_{\widehat{\bf h}_{\varepsilon}} -\widehat{E}_{\widehat{\bf h}}\right) + E_{\widehat{\bf h}_{\varepsilon}} - 2\widehat{E}_{\widehat{\bf h}_{\varepsilon}} + R_2(\widetilde{\bf h}) - R_2({\bf h}^*) - 2\left(\w{R}_2(\widetilde{\bf h}) - \w{R}_2({\bf h}^*)\right).
\end{multline*}

It results that from the above equation and Equation~\eqref{eq:eqDecomp1}, \eqref{eq:bayes-step-p} and \eqref{eq:eqDecomp} that
\begin{multline}\label{eq:First-D-term}
\mathbb{E}\left[R_2(\widehat{{\bf h}}) - R_2({\bf h}^*)\right] \leq \mathbb{E}\left[E_{\widehat{\bf h}} - E_{\widehat{\bf h}_{\varepsilon}}\right] + 2\mathbb{E}\left[\widehat{E}_{\widehat{\bf h}_{\varepsilon}}-\widehat{E}_{\widehat{\bf h}}\right] + \mathbb{E}\left[E_{\widehat{\bf h}_{\varepsilon}} - 2\widehat{E}_{\widehat{\bf h}_{\varepsilon}}\right]\\
+ C\left[\dfrac{\log^4(N)}{D_1^2} + \dfrac{\log^6(N)}{D_2^2} + \dfrac{n}{\log(N)}\exp\left(-c\log^2(N)\right) + \dfrac{1}{N} + \Delta\right].
\end{multline}

We first focus on the upper bound of 
$$\mathbb{E}\left[E_{\widehat{\bf h}} - E_{\widehat{\bf h}_{\varepsilon}}\right].$$
    
From the result of Corollary~\ref{coro:CardH}, we have 
\begin{equation}
    \label{eq:phi-bound1}
    \mathbb{E}\left[E_{\widehat{\bf h}} - E_{\widehat{\bf h}_{\varepsilon}}\right] = \E\left[R_2(\widehat{\bf h})- R_2(\widehat{\bf h}_{\varepsilon})\right] \leq CD_1\log^{3}(N)\varepsilon^{2}.
\end{equation}

We now study the upper bound of 
$$2\mathbb{E}\left[\widehat{E}_{\widehat{\bf h}_{\varepsilon}} - \widehat{E}_{\widehat{\bf h}}\right].$$

We have $\widehat{E}_{\widehat{\bf h}_{\varepsilon}} - \widehat{E}_{\widehat{\bf h}} = \widehat{R}_2(\widehat{\bf h}_{\varepsilon})-\widehat{R}_2(\widehat{\bf h})$, and from Corollary~\ref{coro:CardH}, we obtain
\begin{equation}\label{eq:phi-bound2}
    \E\left[\widehat{E}_{\widehat{\bf h}_{\varepsilon}} - \widehat{E}_{\widehat{\bf h}}\right] \leq C\left(D_1^{1/2}n^{1/2}(\log(N))^{3/2} + D_1\log^3(N)\right)\varepsilon.
\end{equation}

We conclude from Equations~\eqref{eq:First-D-term}, \eqref{eq:phi-bound1} and \eqref{eq:phi-bound2} that
\begin{multline*}
    \mathbb{E}\left[R_2(\widehat{{\bf h}}) - R_2({\bf h}^*)\right]   \leq  CD_1\log^{3}(N)\varepsilon^{2} + CD_1\log^{3}(N)\varepsilon + \mathbb{E}\left[E_{\widehat{\bf h}_{\varepsilon}} - 2\widehat{E}_{\widehat{\bf h}_{\varepsilon}}\right]\\
     \leq CD_1\log^{3}(N)\varepsilon^{2} + CD_1\log^{3}(N)\varepsilon + \mathbb{E}\left[\underset{{\bf h} \in \w{\bf H}^{\varepsilon}_{D_1, D_2}}{\sup}{\left\{E_{\bf h} - 2\widehat{E}_{\bf h}\right\}}\right].
\end{multline*}
Thus, we have:
\begin{equation} \label{eq:upper-bound-T1}
     \mathbb{E}\left[R_2(\widehat{{\bf h}}) - R_2({\bf h}^*)\right] \leq CD_1\log^{3}(N)\varepsilon^{2} + C\left(D_1^{1/2}\log^{3/2}(N)\varepsilon\sqrt{n} + \log^{3}(N)D_1\varepsilon\right) + T_1,
\end{equation}
where $T_1 = \mathbb{E}\left[\underset{{\bf h} \in \w{\bf H}^{\varepsilon}_{D_1, D_2}}{\sup}{\left\{E_{\bf h} - 2\widehat{E}_{\bf h}\right\}}\right]$. For all $u>0$,
\begin{multline*}
    T_1 = \int_{0}^{+\infty}{\mathbb{P}\left(\underset{{\bf h} \in \w{\bf H}^{\varepsilon}_{D_1, D_2}}{\sup}{E_{{\bf h}} - 2\widehat{E}_{{\bf h}}} \geq  t\right)dt}\\
    \leq u+\int_{u}^{+\infty}{\mathbb{P}\left(\underset{{\bf h} \in \w{\bf H}^{\varepsilon}_{D_1, D_2}}{\sup}{E_{{\bf h}} - 2\widehat{E}_{{\bf h}}} \geq  t\right)dt} = u + \int_{u}^{+\infty}{L_tdt}
\end{multline*}

where $L_t=\mathbb{P}\left(\underset{{\bf h} \in \w{\bf H}^{\varepsilon}_{D_1, D_2}}{\sup}{E_{{\bf h}} - 2\widehat{E}_{{\bf h}}} \geq  t\right)$, and for all $t\geq u$,
\begin{align*}
L_t = \mathbb{P}\left(\underset{{\bf h} \in \w{\bf H}_{D_1, D_2}^{\varepsilon}}{\bigcup}{\left\{E_{{\bf h}} - 2\widehat{E}_{{\bf h}} \geq  t\right\}}\right) \leq \sum_{{\bf h} \in \w{\bf H}_{D_1, D_2}^{\varepsilon}}{\mathbb{P}\left(E_{{\bf h}} - 2\widehat{E}_{{\bf h}} \geq  t\right)}.
\end{align*}

For all ${\bf h} \in \w{\bf H}^{\varepsilon}_{D_1, D_2}$, one has:
\begin{align*} 
\widehat{E}_{{\bf h}} = \widehat{R}_2\left({\bf h}\right) - \widehat{R}_2(\widetilde{\bf h}) = \dfrac{1}{N}\sum_{j=1}^{N}{\sum_{k=1}^{K}{\left\{\left(1 - Z^{j}_{k}h^{k}(\bar{X}^{j})\right)^2 - \left(1 - Z^{j}_{k}\widetilde{h}^{k}(\bar{X}^{j})\right)^2\right\}}} = \dfrac{1}{N}\sum_{j=1}^{N}{S_j},
\end{align*}

with
$$
S_j = \sum_{k=1}^{K}{\left\{\left(1 - Z^{j}_{k}h^{k}(\bar{X}^{j})\right)^2 - \left(1 - Z^{j}_{k}\widetilde{h}^{k}(\bar{X}^{j})\right)^2\right\}} = \ell_{\bf h}(Z,X) - \ell_{\widetilde{\bf h}}(Z,X)
$$

where,
$$ \ell_{\bf h}(Z,X) = \sum_{k=1}^{K}{(1 - Z^{j}_{i}h^{k}(\bar{X}^{j}))^2}. $$
We deduce for all $j \in [N]$ that
\begin{multline*}
    \E\left(S^{2}_{j}\right) \leq \E\left[\left(\ell_{\bf h}(Z,X) - \ell_{\widetilde{\bf h}}(Z,X)\right)^2\right] \\
    \leq K\sum_{k=1}^{K}{\E\left[\left(h^{k} - \widetilde{h}^{k}\right)^2(\bar{X}^{j})\left(2 - Z^{j}_{k}\left(h^{k} - \widetilde{h}^{k}\right)(\bar{X}^{j})\right)^2\right]} \\
    \leq 16K\sum_{k=1}^{K}{\E\left[\left(h^{k}- \widetilde{h}^{i}\right)^2(\bar{X}^{j})\right]},
\end{multline*}

since $|Z^{j}_{k}| = 1$ and $|h^{k}|, ~ |\widetilde{h}^{k}| \leq 1$. Then, for all $j \in [N]$ and for all $t \geq \E\left[R_2(\widetilde{\bf h}) - R_2({\bf h}^*)\right]$,
\begin{multline}
\label{eq:Bound-Sj-square}
    \E\left(S^{2}_{j}\right) \leq 16K\sum_{i=1}^{K}{\E\left[\left(h^{k} - \widetilde{h}^{k}\right)^2(\bar{X}^{j})\right]}\\
    \leq 32K\sum_{i=1}^{K}{\E\left[\left(h^{k} - h^{*k}\right)^2(\bar{X}^{j})\right]} + 32K\sum_{i=1}^{K}{\E\left[ \left(\widetilde{h}^{k} - h^{*k}\right)^2(\bar{X}^{j})\right]}\\
    \leq 32K\left[R_2({\bf h}) - R_2({\bf h}^*)\right] + 32K\E\left[R_2(\widetilde{\bf h}) - R_2({\bf h}^*)\right]\\
    \leq 32K(E_{\bf h} + t).
\end{multline}

From Equation~\eqref{eq:Bound-Sj-square}, there exists a constant $C_K>0$ depending on $K$ such that
\begin{align*}
    \forall j\in[\![1,N]\!], \ \ \mathbb{E}\left(S^{2}_{j}\right) \leq &~  C_K(E_{\bf h} + t) = v, ~~~ \forall ~ t \geq \E\left[R_2(\widetilde{\bf h}) - R_2({\bf h}^*)\right].
\end{align*}
Furthermore, there exists a constant $c_K>0$ depending on $K$ such that for all $j \in \{1, \ldots, N\}, ~ |S_j| \leq c_K$. Thus, we obtain from Proposition~\ref{prop:bernstein} that for all ${\bf h} \in \w{\bf H}^{\varepsilon}_{D_1, D_2}$,
\begin{multline*}  
\mathbb{P}\left(E_{\bf h} - 2 \w{E}_{\bf h} \geq t\right) = \mathbb{P}\left(E_{\bf h} - \w{E}_{\bf h} \geq \dfrac{E_{\bf h} + t}{2}\right)\\
\leq \exp\left(-\frac{N\left(E_{\bf h} + t\right)^2/8}{CK\left(E_{\bf h} + t\right)+c_K\left(E_{\bf h} + t\right)}\right)\\
= \exp\left(-\frac{N\left(E_{\bf h} + t\right)}{8(CK+c_K)}\right).
\end{multline*}

Finally, since $\exp\left(-\frac{NE_{\bf h}}{8(CK + c_K)}\right) \leq  1$, we obtain that for all ${\bf h} \in \w{\bf H}^{\varepsilon}_{D_1, D_2}$,
\begin{align*} 
\forall ~ t \geq \E\left[R_2(\widetilde{\bf h}) - R_2({\bf h}^*)\right], ~~ \mathbb{P}\left(E_{\bf h} - 2 \w{E}_{\bf h} \geq t\right)  \leq \exp\left(-\frac{Nt}{8(CK + c_K)}\right).
\end{align*}
It results that for all $t \geq \E\left[R_2(\widetilde{\bf h}) - R_2({\bf h}^*)\right]$, 
\begin{equation*}
    L_t \leq {\rm Card}\left(\w{\bf H}_{D_1, D_2}\right) \exp\left(-\frac{Nt}{8(CK+c_K)}\right) = \mathcal{N}\left(\varepsilon, \mathcal{S}_{D_1}^K \times \widetilde{\mathcal{S}}_{D_2}, \|\cdot\|_{\infty}\right)\exp\left(-\frac{Nt}{8(CK+c_K)}\right),
\end{equation*} 

and for all $u \geq \E\left[R_2(\widetilde{\bf h}) - R_2({\bf h}^*)\right]$,
\begin{multline*}
    T_1 \leq u+\int_{u}^{+\infty}{L_tdt} = u + \mathcal{N}\left(\varepsilon, \mathcal{S}_{D_1}^K \times \widetilde{\mathcal{S}}_{D_2}, \|\cdot\|_{\infty}\right)\int_{u}^{+\infty}{\exp\left(-\frac{Nt}{8(CK+c_K)}\right)dt}\\
    \leq u + \frac{8(CK+c_K)\mathcal{N}\left(\varepsilon, \mathcal{S}_{D_1}^K \times \widetilde{\mathcal{S}}_{D_2}, \|\cdot\|_{\infty}\right)}{N}\exp\left(-\frac{Nu}{8(CK+c_K)}\right).
\end{multline*}

From Equation~\eqref{eq:bayes-step-p}, we set 
\begin{multline*}
    u = \E\left[R_2(\widetilde{\bf h}) - R_2({\bf h}^*)\right] + \dfrac{8(CK+c)\log\left[\mathcal{N}\left(\varepsilon, \mathcal{S}_{D_1}^K \times \widetilde{\mathcal{S}}_{D_2}, \|\cdot\|_{\infty}\right)\right]}{N} \\
    \leq C\left[\dfrac{\log^4(N)}{D_1^2} + \dfrac{\log^6(N)}{D_2^2} + \dfrac{n\exp\left(-c\log^2(N)\right)}{\log(N)} + \dfrac{1}{N} + \Delta + \dfrac{\log\left[\mathcal{N}\left(\varepsilon, \mathcal{S}_{D_1}^K \times \widetilde{\mathcal{S}}_{D_2}, \|\cdot\|_{\infty}\right)\right]}{N} \right], 
\end{multline*} 

and we deduce that 
\begin{align*} 
T_1 \leq C\left(\dfrac{\log^4(N)}{D_1^2} + \dfrac{\log^6(N)}{D_2^2} + \dfrac{n}{\log(N)}\exp\left(-c\log^2(N)\right) + \dfrac{1}{N} + \Delta + \dfrac{\log\left[\mathcal{N}\left(\varepsilon, \mathcal{S}_{D_1}^K \times \widetilde{\mathcal{S}}_{D_2}, \|\cdot\|_{\infty}\right)\right]}{N}\right).
\end{align*}

From Proposition~\ref{prop:EpsNet} and for $\varepsilon = \dfrac{\sqrt{D_{\max} + M} \log^3(N)}{N^3}$ and $D_{\max} \leq N$, there exists a constant $C > 0$ depending on $K$ such that
\begin{equation}\label{eq:Upper-Bound-T1-2}
    T_1 \leq C\left(\dfrac{\log^4(N)}{D_1^2} + \dfrac{\log^6(N)}{D_2^2} + \dfrac{D_{1}\log(N)}{N} + \dfrac{D_{2}\log(N)}{N} + \dfrac{n\exp\left(-c\log^2(N)\right)}{\log(N)} + \dfrac{1}{N} + \Delta\right).
\end{equation}

Finally, for $\Delta = \mathrm{O}(1/N)$, it follows from Equations~\eqref{eq:upper-bound-T1} and \eqref{eq:Upper-Bound-T1-2} that,
\begin{multline*}
\mathbb{E}\left[R_2(\widehat{{\bf h}}) - R_2({\bf h}^*)\right] \leq C\left(\dfrac{\log^4(N)}{D_1^2} + \dfrac{\log^6(N)}{D_2^2} + \dfrac{D_{1}\log(N)}{N} + \dfrac{D_{2}\log(N)}{N}\right)\\
+ C\left(\dfrac{n}{\log(N)}\exp\left(-c\log^2(N)\right) + \dfrac{1}{N} + \Delta\right)\\
\leq C\left(\dfrac{\log^4(N)}{D_1^2} + \dfrac{\log^6(N)}{D_2^2} + \dfrac{D_{1}\log(N)}{N} + \dfrac{D_{2}\log(N)}{N} + \Delta\right).
\end{multline*}
\end{proof}

\subsection{Proof of Corollary~\ref{coro:Consistency}}

\begin{proof}
    We recall the following relation between the excess classification risk of $\w{g} = g_{\w{\bf h}}$ and the excess risk of $\w{\bf h}$ provided by Proposition~\ref{prop:RelatingExcessRisks}:
    \begin{equation*}
        R(\w{g}) - R(g^*) = R(g_{\w{\bf h}}) - R(g_{\bf h^{*}})\leq\frac{1}{\sqrt{2}}\left({R_2}(\w{\bf h})-{R_2}(\bf h^{*})\right)^{1/2}.
    \end{equation*}
    
    We deduce that
    \begin{equation}\label{eq:ERisk1}
        \mathbb{E}\left[R(\w{g}) - R(g^*)\right] \leq \left(\mathbb{E}\left[{R_2}(\w{\bf h})-{R_2}(\bf h^{*})\right]\right)^{1/2}.
    \end{equation}
    
    Then, from Theorem~\ref{thm:EstimErrorERM} and Equation~\eqref{eq:ERisk1}, there exists a constant $C>0$ such that
    \begin{equation*}
        \mathbb{E}\left[R(\w{g}) - R(g^*)\right] \leq C\left(\log^{3}(N)\left[\dfrac{1}{D_1} + \dfrac{1}{D_2}\right] + \left(\dfrac{D_1\log(N)}{N}\right)^{1/2} + \left(\dfrac{D_2\log(N)}{N}\right)^{1/2} + \sqrt{\Delta}\right).
    \end{equation*}
    
    For $D_1 = D_2 = D_N$ and $n = \mathrm{O}(1/N)$, there exists a new constant $C>0$ such that
    \begin{equation*}
        \mathbb{E}\left[R(\w{g}) - R(g^*)\right] \leq C\left(\left(\dfrac{\log^{6}(N)}{D_N^2}\right)^{1/2} + \left(\dfrac{D_N\log(N)}{N}\right)^{1/2} + N^{-1/2}\right).
    \end{equation*}
    
    We finally obtain the consistency of the ERM classifier $\widehat{g} = g_{\w{\bf h}}$ under the assumptions
    \begin{equation*}
        \dfrac{\log^{6}(N)}{D_N^2} \longrightarrow 0, ~~ \dfrac{D_N\log(N)}{N} \longrightarrow 0 ~~ \mathrm{as} ~~ N \rightarrow \infty.
    \end{equation*}
\end{proof}

\section{Proofs of Section~\ref{sec:RatesConv}}

\subsection{Proof of Corollary~\ref{coro:RateMA}}

\begin{proof}
    The result of Theorem only focus on the set of Lipschitz functions. If we restrict the set of Lipschitz functions to the H\"older spaces $\Sigma(\beta_1, R)^K$ and $\Sigma(\beta_2, R)$ respectively for the functions ${\bf b}^* = \left(b_1^*, \ldots, b_K^*\right)$ and $\sigma^{*2}$, then, for $\Delta = \mathrm{O}(1/N)$, we deduce from the proof of Theorem~\ref{thm:EstimErrorERM} (see Equation~\eqref{eq:Bound-Term3}) that
    \begin{multline}\label{eq:ExcessRiskHolder}
        \mathbb{E}\left[{R_2}(\w{\bf h})-{R_2}(\bf h^{*})\right]\\
        \leq C \log^4(N)\left( \underset{k \in [K]}{\max}{\underset{f \in \mathcal{S}_{D_1}}{\inf}{\left\|f - b_k^*\right\|_{n}^2}} + \underset{f \in \widetilde{\mathcal{S}}_{D_2}}{\inf}{\left\|f - \sigma^{*2}\right\|_{n}^2}\right) + C\left(\dfrac{D_1\log(N)}{N} + \dfrac{D_2\log(N)}{N} + \Delta\right).
    \end{multline}
    
    For $D_1 = D_{N}^1$ and $D_2 = D_{N}^2$ and under Assumptions~\ref{ass:RegEll} and \ref{ass:Novikov}, we obtain from the result of Proposition~\ref{prop:RelatingExcessRisks} together with Equation~\eqref{eq:ExcessRiskHolder} that
    \begin{multline*}
        \mathbb{E}\left[R(\w{g}) - R(g^*)\right] = \mathbb{E}\left[R(g_{\w{\bf h}}) - R(g_{\bf h^{*}})\right] \leq \dfrac{1}{\sqrt{2}}\left({R_2}(\w{\bf h})-{R_2}(\bf h^{*})\right)^{1/2}\\
        \leq C\log^2(N) \left[ \left(\underset{k \in [K]}{\max}{\underset{f \in \mathcal{S}_{D_1}}{\inf}{\left\|f - b_k^*\right\|_{n}^2}}\right)^{1/2} + \left(\underset{f \in \widetilde{\mathcal{S}}_{D_2}}{\inf}{\left\|f - \sigma^{*2}\right\|_{n}^2}\right)^{1/2}\right]\\
        + C\left[\left(\dfrac{D_{N}^1\log(N)}{N}\right)^{1/2} + \left(\dfrac{D_{N}^2\log(N)}{N}\right)^{1/2} + \sqrt{\Delta}\right].
    \end{multline*}
    
    From \cite{denis2020ridge}, \textit{Lemma D.2}, there exist constants $C_1, C_2 > 0$ such that
    \begin{align*}
        \forall ~ k \in [K], ~~ &\underset{f \in \mathcal{S}_{D_1}}{\inf}{\left\|f - b_k^*\right\|_{n}^2} \leq C_1\left(\dfrac{A_{N}}{D_{N}^1}\right)^{2\beta_1}\\
        & \underset{f \in \widetilde{\mathcal{S}}_{D_2}}{\inf}{\left\|f - \sigma^{*2}\right\|_{n}^2} \leq C_2\left(\dfrac{A_N}{D_{N}^2}\right)^{2\beta_2}.
    \end{align*}
    
    Then, for $n = \mathrm{O}(1/N), ~ A_N = \log(N)$, $D_{N}^1 = N^{1/(2\beta_1+1)}$ and $D_{N}^2 = N^{1/(2\beta_2+1)}$, we deduce that
    \begin{equation*}
        R(\w{g}) - R(g^*) \leq C\left(\log^{\beta_1+2}(N)N^{-\beta_1/(2\beta_1+1)} + \log^{\beta_2+2}(N)N^{-\beta_2/(2\beta_2+1)}\right).
    \end{equation*}
\end{proof}

\subsection{Proof of Proposition~\ref{prop:TheMargin}}

\begin{proof}
Define for all function $h \in \mathbb{L}^{2}(\R)$ and from the diffusion process $X = (X_t)_{t \in [0,1]}$ the following pseudo-norm:
\begin{equation*}
    \|h\|^{2}_{X} := \int_{0}^{1}{h^{2}(X_s)ds}.
\end{equation*}
Then, from Assumption~\ref{ass:Margin}, we obtain that
$$ \underline{\Delta}_{b} \leq \left\|b^{*}_{1} - b^{*}_{0}\right\|_{X} \leq \overline{\Delta}_{b} $$
where $\underline{\Delta}_{b}$ and $\overline{\Delta}_{b}$ are positive real number such that $\overline{\Delta}_{b} > \underline{\Delta}_{b}$. Let $m \geq 1$ be an integer. Consider the subdivision 
$$\Gamma_{m} := \left\{t_{k} = \underline{\Delta}_{b} + k\frac{\overline{\Delta}_{b} - \underline{\Delta}_{b}}{m}, ~~ k = 0, \ldots, m\right\}$$

of the compact interval $[\underline{\Delta}_{b}, \overline{\Delta}_{b}]$. Set $f^{*} = b^{*}_{1} - b^{*}_{0}$, which implies that $\Delta^{2}_{b^{*}} = \E\left[\int_{0}^{1}{f^{*2}(X_s)ds}\right]$. 
Then, from the total probability formula, we have
\begin{align*}
    \P\left(0 < \left|\pi_1^*(X) - \frac{1}{2}\right| \leq \varepsilon\right) = &~ \sum_{k=0}^{m-1}{\P\left(0 < \left|\pi_1^*(X) - \frac{1}{2}\right| \leq \varepsilon \biggm\vert \|f^{*}\|_{X} \in [t_k,t_{k+1}]\right)\P\left(\|f^{*}\|_{X} \in [t_k,t_{k+1}]\right)}.
\end{align*}

For all $k \in [\![0,m-1]\!]$,
\begin{align*}
     \P\left(0 < \left|\pi_1^*(X) - \frac{1}{2}\right| \leq \varepsilon \biggm\vert \|f^{*}\|_{X} \in [t_k,t_{k+1}]\right) = &~ \P\left(0 < \frac{|Q_1(X) - Q_0(X)|}{2(Q_0(X) + Q_1(X))} \leq \varepsilon \biggm\vert \|f^{*}\|_{X} \in [t_k,t_{k+1}]\right) \\
     = &~ T^{k}_{1} + T^{k}_{2}
\end{align*}
 where,
\begin{align*}
    T^{k}_1 = &~ \P\left(\left\{Q_1(X) < Q_0(X)\right\} \cap \left\{\frac{|Q_1(X) - Q_0(X)|}{2(Q_0(X) + Q_1(X))} \leq \varepsilon\right\} \biggm\vert \|f^{*}\|_{X} \in [t_k,t_{k+1}]\right) \\
    T^{k}_2 = &~ \P\left(\left\{Q_1(X) > Q_0(X)\right\} \cap \left\{\frac{|Q_1(X) - Q_0(X)|}{2(Q_0(X) + Q_1(X))} \leq \varepsilon\right\} \biggm\vert \|f^{*}\|_{X} \in [t_k,t_{k+1}]\right).
\end{align*}

\paragraph{Upper-bound of $T^{k}_1$.}
The first term $T_1$ satisfies
\begin{align*}
    T^{k}_1 \leq &~ \P\left(\left\{Q_1(X) < Q_0(X)\right\} \cap \left\{\left|\frac{Q_1(X)}{Q_0(X)} - 1\right| \leq 4\varepsilon \right\} \biggm\vert \|f^{*}\|_{X} \in [t_k,t_{k+1}]\right) \\
    = &~ \frac{1}{2}T^{k}_{1,1} + \frac{1}{2}T^{k}_{1,2}
\end{align*}

where,
\begin{align*}
    T^{k}_{1,1} = &~ \P\left(\left\{Q_1(X) < Q_0(X)\right\} \cap \left|\frac{Q_1(X)}{Q_0(X)} - 1\right| \leq 4\varepsilon \biggm\vert \|f^*\|_{X} \in [t_k,t_{k+1}], Y = 0 \right), \\
    T^{k}_{1,2} = &~ \P\left(\left\{Q_1(X) < Q_0(X)\right\} \cap \left|\frac{Q_1(X)}{Q_0(X)} - 1\right| \leq 4\varepsilon \biggm\vert \|f^*\|_{X} \in [t_k,t_{k+1}], Y = 1 \right).
\end{align*}

On the event $\{Y = 0\}$, we have
\begin{align*}
    \frac{Q_1}{Q_0}(X) = &~ \exp\left(\int_{0}^{1}{(b^{*}_{1} - b^{*}_{0})(X_s)b^{*}_{0}(X_s)ds} - \frac{1}{2}\int_{0}^{1}{(b^{*}_{1} - b^{*}_{0})(b^{*}_{0} + b^{*}_{1})(X_s)ds} + \int_{0}^{1}{(b^{*}_{1} - b^{*}_{0})(X_s)dW_s}\right) \\
    = &~ \exp\left(-\frac{1}{2}\int_{0}^{1}{f^{*2}(X_s)ds} + \int_{0}^{1}{f^{*}(X_s)dW_s}\right) \\
    = &~ \exp\left(-\frac{1}{2}\left\|f^{*}\right\|^{2}_{X} + \int_{0}^{1}{f^{*}(X_s)dW_s}\right).
\end{align*}
On the event $\{Y = 1\}$, we have
\begin{align*}
    \frac{Q_1}{Q_0}(X) = &~ \exp\left(\int_{0}^{1}{(b^{*}_{1} - b^{*}_{0})(X_s)b^{*}_{1}(X_s)ds} - \dfrac{1}{2}\int_{0}^{1}{(b^{*}_{1} - b^{*}_{0})(b^{*}_{0} + b^{*}_{1})(X_s)ds} + \int_{0}^{1}{(b^{*}_{1} - b^{*}_{0})(X_s)dW_s}\right) \\
    = &~ \exp\left(\dfrac{1}{2}\int_{0}^{1}{f^{*2}(X_s)ds} + \int_{0}^{1}{f^{*}(X_s)dW_s}\right) \\
    = &~ \exp\left(\dfrac{1}{2}\left\|f^{*}\right\|^{2}_{X} + \int_{0}^{1}{f^{*}(X_s)dW_s}\right).
\end{align*}

We deduce for all $\varepsilon \in (0, 1/8)$ and for all $k \in [\![0,m-1]\!]$ that
\begin{align*}
    T^{k}_{1,1} = &~ \P\left(\left|\exp\left(-\frac{1}{2}\left\|f^{*}\right\|^{2}_{X} + \int_{0}^{1}{f^{*}(X_s)dW_s}\right) - 1\right| \leq 4\varepsilon \biggm\vert \|f^{*}\|_{X} \in [t_k,t_{k+1}], Y = 0\right) \\
    \leq &~ \P\left(\log(1 - 4\varepsilon) \leq \int_{0}^{1}{f^{*}(X_s)dW_s} - \frac{\|f^{*}\|^{2}_{X}}{2} \leq \log(1 + 4\varepsilon) \biggm\vert\|f^{*}\|_{X} \in [t_k,t_{k+1}], Y = 0\right) \\
    \leq &~ \P\left( -\frac{8\varepsilon}{\Delta_{b^{*}}} + \frac{\|f^{*}\|^{2}_{X}}{2\Delta_{b^{*}}} \leq \int_{0}^{1}{\frac{f^{*}(X_s)}{\Delta_{b^{*}}}dW_s} \leq \frac{4\varepsilon}{\Delta_{b^{*}}} + \frac{\|f^{*}\|^{2}_{X}}{2\Delta_{b^{*}}} \biggm\vert \|f^{*}\|_{X} \in [t_k,t_{k+1}], Y = 0\right) \\
    \leq &~ \P\left( -\frac{8\varepsilon}{\Delta_{b^{*}}} + \frac{t^{2}_k}{2\Delta_{b^{*}}} \leq \int_{0}^{1}{\frac{f^{*}(X_s)}{\Delta_{b^{*}}}dW_s} \leq \frac{4\varepsilon}{\Delta_{b^{*}}} + \frac{t^{2}_{k+1}}{2\Delta_{b^{*}}} \biggm\vert \|f^{*}\|_{X} \in [t_k,t_{k+1}], Y = 0\right),
\end{align*}
and
\begin{align*}
    T^{k}_{1,2} = &~ \P\left(\left|\exp\left(\frac{1}{2}\left\|f^{*}\right\|^{2}_{X} + \int_{0}^{1}{f^{*}(X_s)dW_s}\right) - 1\right| \leq 4\varepsilon \biggm\vert \|f^{*}\|_{X} \in [t_k,t_{k+1}], Y = 1\right) \\
    \leq &~ \P\left(\log(1 - 4\varepsilon) \leq \int_{0}^{1}{f^{*}(X_s)dW_s} + \frac{\|f^{*}\|^{2}_{X}}{2}\leq \log(1 + 4\varepsilon) \biggm\vert \|f^{*}\|_{X} \in [t_k,t_{k+1}], Y = 1\right) \\
    \leq &~ \P\left( -\frac{8\varepsilon}{\Delta_{b^{*}}} - \frac{\|f^{*}\|^{2}_{X}}{2\Delta_{b^{*}}} \leq \int_{0}^{1}{\frac{f^{*}(X_s)}{\Delta_{b^{*}}}dW_s} \leq \frac{4\varepsilon}{\Delta_{b^{*}}} - \frac{\|f^{*}\|^{2}_{2}}{2\Delta_{b^{*}}} \biggm\vert \|f^{*}\|_{X} \in [t_k,t_{k+1}], Y = 1\right) \\
    \leq &~ \P\left( -\frac{8\varepsilon}{\Delta_{b^{*}}} - \frac{t^{2}_{k+1}}{2\Delta_{b^{*}}} \leq \int_{0}^{1}{\frac{f^{*}(X_s)}{\Delta_{b^{*}}}dW_s} \leq \frac{4\varepsilon}{\Delta_{b^{*}}} - \frac{t^{2}_k}{2\Delta_{b^{*}}} \biggm\vert \|f^{*}\|_{X} \in [t_k,t_{k+1}], Y = 1\right).
\end{align*}
Under Assumption~\ref{ass:Margin}, the random variable $Z = (1/\Delta_{b^{*}})\int_{0}^{1}{f^{*}(X_s)dW_s}$ admits a density function $p_{Z}$ that is bounded on the real line $\R$. Thus, there exists a constant $C_Z > 0$ such that
\begin{equation*}
    \forall~ x\in\R, ~~ p_Z(x) \leq C_Z.
\end{equation*}
We finally obtain that for all $k \in [\![0,m-1]\!]$,
\begin{align*}
    T^{k}_{1,1} \leq &~ \int_{-\frac{8\varepsilon}{\Delta_{b^{*}}}+\frac{t^{2}_k}{2\Delta_{b^{*}}}}^{\frac{4\varepsilon}{\Delta_{b^{*}}}+\frac{t^{2}_{k+1}}{2\Delta_{b^{*}}}}{p_Z(x)dx} \leq C_Z\left(\frac{12}{\Delta_{b^{*}}}\varepsilon + \frac{1}{2}(t^{2}_{k+1} - t^{2}_k)\right) \leq C\left(\frac{12}{\Delta_{b^{*}}}\varepsilon + \frac{\overline{\Delta}_{b} - \underline{\Delta}_{b}}{2m}\right) \\
    T^{k}_{1,2} \leq &~ \int_{-\frac{8\varepsilon}{\Delta_{b^{*}}}-\frac{t^{2}_{k+1}}{2\Delta_{b^{*}}}}^{\frac{4\varepsilon}{\Delta_{b^{*}}}-\frac{t^{2}_k}{2\Delta_{b^{*}}}}{p_Z(x)dx} \leq C_Z\left(\frac{12}{\Delta_{b^{*}}}\varepsilon + \frac{1}{2}(t^{2}_{k+1} - t^{2}_k)\right) \leq C\left(\frac{12}{\Delta_{b^{*}}}\varepsilon + \frac{\overline{\Delta}_{b} - \underline{\Delta}_{b}}{2m}\right) 
\end{align*}
where $C>0$ is a new constant. Then, for all $k \in [\![0,m-1]\!]$,
$$ T^{k}_{1} \leq C\left(\frac{12}{\Delta_{b^{*}}}\varepsilon + \frac{\overline{\Delta}_{b} - \underline{\Delta}_{b}}{2m}\right) . $$
Using a similar reasoning, we obtain that for all $k \in [\![0,m-1]\!]$, 
$$ T^{k}_{2} \leq C\left(\frac{12}{\Delta_{b^{*}}}\varepsilon + \frac{\overline{\Delta}_{b} - \underline{\Delta}_{b}}{2m}\right) . $$
Thus, we conclude that for all $m \geq 1$,
\begin{align*}
    \P\left(0 < \left|\pi_1^*(X) - \frac{1}{2}\right| \leq \varepsilon\right) = &~ \sum_{k=0}^{m-1}{\P\left(0 < \left|\pi_1^*(X) - \frac{1}{2}\right| \leq \varepsilon \biggm\vert \|f^{*}\|_{X} \in [t_k,t_{k+1}]\right)\P\left(\|f^{*}\|_{X} \in [t_k,t_{k+1}]\right)} \\
    \leq &~ \sum_{k=0}^{m-1}{\left(T^{k}_{1} + T^{k}_{2}\right)\P\left(\|f^{*}\|_{X} \in [t_k,t_{k+1}]\right)} \\
    \leq &~ C\frac{12}{\Delta_{b^{*}}}\varepsilon + \frac{\overline{\Delta}_{b} - \underline{\Delta}_{b}}{2m}.
\end{align*}
Finally, we tend $m$ toward infinity and obtain
\begin{align*}
    \P\left(0 < \left|\pi_1^*(X) - \frac{1}{2}\right| \leq \varepsilon\right) \leq &~ C\frac{12}{\Delta_{b^{*}}}\varepsilon.
\end{align*}
\end{proof}

\subsection{Proof of Theorem~\ref{thm:RateMargin}}

\begin{proof}
    Since we have $b_0^*, b_1^* \in \Sigma(\beta, R)$ and $\Delta = \mathrm{O}(1/N)$, from the proof of Theorem~\ref{thm:EstimErrorERM}, we obtain
    \begin{align*}
    \mathbb{E}\left[R_2(\widehat{\bf h}) - R_2({\bf h}^{*})\right] \leq C\left[ \log^{4}(N)\left(\dfrac{A_N}{D_N}\right)^{2\beta} + \frac{D_N\log(N)}{N} + \dfrac{1}{N} \right].
   \end{align*}
   
   Then, for $D_N \propto N^{1/(2\beta+1)}$ and $A_N = \log(N)$, there exists a constant $C>0$ such that
   \begin{align*}
    \mathbb{E}\left[R_2(\widehat{\bf h}) - R_2({\bf h}^{*})\right]  \leq C\log^{2\beta+4}(N)N^{-2\beta/(2\beta+1)}.
   \end{align*}
   
    Finally, from Equation~\eqref{eq:LinkERmargin}, we conclude that
   \begin{equation*}
      \mathbb{E}\left[R(\widehat{g}) - R(g^{*})\right] \leq \left(\mathbb{E}\left[R_2(\widehat{\bf h}) - R_2({\bf h}^{*})\right]\right)^{2/3} \leq C\log^{(4\beta+8)/3}(N)N^{-4\beta/3(2\beta+1)}.
   \end{equation*}
\end{proof}

\section{Proofs of Section~\ref{sec:AdaptiveERM}}

\subsection{Proof of Theorem~\ref{thm:AdaptiveERMclassifier}}

\begin{proof}
    The adaptive ERM classifier $g_{\w{\bf h}}$ defined from the score function $\w{\bf h} = \w{\bf h}_{\w{D}_N^1, \w{D}_N^2}$ given by
    \begin{equation}\label{eq:MinPen}
        \w{R}_2(\w{\bf h}) + \mathrm{pen}(\w{D}_N^1, \w{D}_N^2) = \underset{D_N^1, D_N^2 \in \Xi_N}{\inf}{\left\{~\w{R}_2(\w{\bf h}_{D_N^1, D_N^2}) + \mathrm{pen}(D_N^1, D_N^2)\right\}}.
    \end{equation}
    
    For all $D_N^1, D_N^2 \in \Xi_N$, we have
    \begin{equation}\label{eq:First-UB-ExcessRisk}
        \E\left[R_2(\w{\bf h}) - R_2({\bf h}^{*})\right] \leq \E\left[\w{T}_{1,N}\right] + \E\left[\w{T}_{2,N}\right]
    \end{equation}
    
    with,
    \begin{align*}
        \w{T}_{1,N} = &~ \E\left[R_2(\w{\bf h}) - R_2({\bf h}^{*})\right] -2\left(\w{R}_2(\w{\bf h}) - \w{R}_2({\bf h}^{*})\right) - 2\mathrm{pen}(\w{D}_N^1, \w{D}_N^2), \\
        \w{T}_{2,N} = &~ 2\left(\w{R}_2(\w{\bf h}) - \w{R}_2({\bf h}^{*})\right) + 2 \mathrm{pen}(\w{D}_N^1, \w{D}_N^2).
    \end{align*}

\paragraph{Upper-bound of the second term $\mathbb{E}\left[\w{T}_{2,N}\right]$.}

   From Equation~\eqref{eq:MinPen}, we have
    \begin{multline*}
        \w{T}_{2,N} = 2\left(\w{R}_2(\w{\bf h}) + \mathrm{pen}(\w{D}_N^1, \w{D}_N^2)\right) - 2 \w{R}_2({\bf h}^{*}) \\
        = 2\underset{D_N^1, D_N^2 \in \Xi_N}{\inf}{\left\{\w{R}_2(\w{\bf h}_{D_N^1, D_N^2}) + \mathrm{pen}(D_N^1, D_N^2)\right\}} - 2\w{R}_2({\bf h}^{*}) \\
        = 2\underset{D_N^1, D_N^2 \in \Xi_N}{\inf}{\left\{\underset{{\bf h} \in \w{\bf H}_{D_N^1, D_N^2}}{\inf}{\left\{\w{R}_2({\bf h}) - \w{R}_2({\bf h}^{*})\right\}} + \mathrm{pen}(D_N^1, D_N^2)\right\}}.
    \end{multline*}
    
    Thus, since for all $D_1, D_2 \in \Xi_N$, the random space $\w{\bf H}_{D_1,D_2}$ depends on the sample $\mathcal{Y}_N = \left\{\widetilde{Y}_1, \ldots, \widetilde{Y}_N\right\}$, and the empirical risk $\w{R}_2$ is built from the learning sample $\mathcal{D}_N$ which is independent of $\mathcal{Y}_N$, we obtain
    \begin{multline}\label{eq:UB-T2N}
        \E\left[\w{T}_{2,N}\right] \leq 2\underset{D_N^1, D_N^2 \in \Xi_N}{\inf}{\left\{\E\left[\underset{{\bf h} \in \w{\bf H}_{D_N^1, D_N^2}}{\inf}{\left\{\w{R}_2({\bf h}) - \w{R}_2({\bf h}^{*})\right\}}\right] + \mathrm{pen}(D_N^1, D_N^2)\right\}} \\
        \leq 2\underset{D_N^1, D_N^2 \in \Xi_N}{\inf}{\left\{\mathbb{E}\left[R_2(\widetilde{\bf h}) - R_2({\bf h}^{*})\right] + \mathrm{pen}(D_N^1, D_N^2)\right\}},
    \end{multline}

where,
\begin{equation*}
\widetilde{{\bf h}} \in \argmin{{\bf h} \in \w{\bf H}_{D_1,D_2}} R_2({\bf h}).
\end{equation*}

\paragraph{Upper-bound of the first term $\mathbb{E}\left[\w{T}_{1,N}\right]$.}

We have on one side:
    \begin{multline*}
        \w{T}_{1,N} = \E\left[R_2(\w{\bf h}) - R_2({\bf h}^*)\right] - 2\left[\w{R}_2(\w{\bf h}) - \w{R}_2({\bf h}^*)\right] - 2\mathrm{pen}(\w{D}_N^1, \w{D}_N^2) \\
        = \E\left[\mathcal{E}_{\w{\bf h}} - \mathcal{E}_{\w{\bf h}_{\varepsilon}}\right] - 2\left(\w{\mathcal{E}}_{\w{\bf h}} - \w{\mathcal{E}}_{\w{\bf h}_{\varepsilon}}\right) + \w{T}^{\varepsilon}_{1,N}\\
        = \E\left[R_2(\w{\bf h}) - R_2(\w{\bf h}_{\varepsilon})\right] - 2\left(\w{R}_2(\w{\bf h}) - \w{R}_2(\w{\bf h}_{\varepsilon})\right) + \w{T}_{1,N}^{\varepsilon},
    \end{multline*}
    
    with 
    $$\w{T}^{\varepsilon}_{1,N} = \E\left[\mathcal{E}_{\w{\bf h}_{\varepsilon}}\right] - 2\w{\mathcal{E}}_{\w{\bf h}_{\varepsilon}} - 2\mathrm{pen}(\w{D}_N^1, \w{D}_N^2),$$
    
    and
    \begin{equation*}
       \mathcal{E}_{\bf h} = R_2({\bf h}) - R_2({\bf h}^{*}), ~ \w{\mathcal{E}}_{\bf h} = \w{R}_2({\bf h}) - \w{R}_2({\bf h}^{*}).
   \end{equation*}
    
    From Corollary~\ref{coro:CardH}, there exists a constant $C>0$ such that
    \begin{equation}\label{eq:upper-bound-T1N}
        \E\left[\w{T}_{1,N}\right] \leq C\left(N^{1/2}n^{1/2}(\log(N))^{3/2} + N\log^3(N)\right)\varepsilon + \E\left[\w{T}^{\varepsilon}_{1,N}\right].
    \end{equation}
    
    On the other side, one has
    \begin{equation*}
        \E\left[\w{T}^{\varepsilon}_{1,N}\right] = \E\left[\w{T}^{\varepsilon}_{1,N}\one_{\w{T}^{\varepsilon}_{1,N} > 0}\right] + \E\left[\w{T}^{\varepsilon}_{1,N}\one_{\w{T}^{\varepsilon}_{1,N} \leq 0}\right] \leq \int_{0}^{+\infty}{\P\left(\w{T}^{\varepsilon}_{1,N} > t\right)dt},
    \end{equation*}
    
    and for all $t > 0$,
    \begin{multline*}
        \P\left(\w{T}^{\varepsilon}_{1,N} > t\right) = \P\left(\E\left[\mathcal{E}_{\w{\bf h}_{\varepsilon}}\right] - 2\w{\mathcal{E}}_{\w{\bf h}_{\varepsilon}} > t + 2 \mathrm{pen}(\w{D}_N^1, \w{D}_N^2)\right) \\
        \leq \sum_{D_N^1, D_N^2 \in \Xi_N}{\P\left(\underset{{\bf h} \in \w{\bf H}^{\varepsilon}_{D_N^1, D_N^2}}{\sup}{\left\{~\mathcal{E}_{h} - 2\w{\mathcal{E}}_{\bf h}\right\} > t + 2 \mathrm{pen}(D_N^1, D_N^2)}\right)} \\
        \leq \sum_{D_N^1, D_N^2 \in \Xi_N}{\sum_{{\bf h} \in \w{\bf H}^{\varepsilon}_{D_N^1, D_N^2}}{\P\left(\mathcal{E}_{\bf h} - \w{\mathcal{E}}_{\bf h} > \dfrac{t + 2 \mathrm{pen}(D_N^1, D_N^2) + \mathcal{E}_{\bf h}}{2}\right)}}.
    \end{multline*}

    We also obtain from the proof of Theorem~\ref{thm:EstimErrorERM} that for all $D_N^1, D_N^2 \in \Xi_N$ and for all ${\bf h} \in \w{\bf H}^{\varepsilon}_{D_N^1, D_N^2}$, there exist constants $C_1,C_2>0$ such that 
    \begin{multline*}
        \P\left(\mathcal{E}_{\bf h} - \w{\mathcal{E}}_{\bf h} > \dfrac{t + 2 \mathrm{pen}(D_N^1, D_N^2) + \mathcal{E}_{\bf h}}{2}\right) \leq C_1\exp\left(-C_2N\left[\mathcal{E}_{\bf h} + t + 2\mathrm{pen}(D_N^1, D_N^2)\right]\right) \\
        \leq C_1\exp\left(-C_2N\left[ t + 2\mathrm{pen}(D_N^1, D_N^2)\right]\right).
    \end{multline*}
    
   Since from Corollary~\ref{coro:CardH},
   \begin{equation*}
    {\rm Card}\left(\w{\bf H}_{D_N^1, D_N^2}^{\varepsilon}\right) = \mathcal{N}(\varepsilon, \mathcal{S}_{D_N^1}^{K} \times \widetilde{\mathcal{S}}_{D_N^2}, \|.\|_{\infty}),
\end{equation*}
   
   we finally obtain that:
    \begin{multline}\label{eq:Teps1}
        \E\left[\w{T}^{\varepsilon}_{1,N}\right] \leq C\sum_{D_N^1, D_N^2 \in \Xi_N}{\mathcal{N}(\varepsilon, \mathcal{S}_{D_N^1}^{K} \times \widetilde{\mathcal{S}}_{D_N^2}, \|.\|_{\infty})\int_{0}^{+\infty}{\exp\left(-C_2N\left[t+2\mathrm{pen}(D_N^1, D_N^2)\right]\right)dt}} \\
        \leq C\sum_{D_N^1, D_N^2 \in \Xi_N}{\dfrac{\mathcal{N}(\varepsilon, \mathcal{S}_{D_N^1}^{K} \times \widetilde{\mathcal{S}}_{D_N^2}, \|.\|_{\infty})\exp\left(-2C_2N\mathrm{pen}(D_N^1, D_N^2)\right)}{C_2N}}.
    \end{multline}

from Proposition~\ref{prop:EpsNet}, there exists a constant $C>0$ depending on $K$ such that
\begin{equation*}
    \log\left(\mathcal{N}(\varepsilon, \mathcal{S}_{D_N^1}^{K} \times \widetilde{\mathcal{S}}_{D_N^2}, \|.\|_{\infty})\right) \leq C\left(D_N^1 + D_N^2\right)\log\left(\dfrac{\sqrt{D_N^1 + D_N^2 + M}\log^3(N)}{\varepsilon}\right).
\end{equation*}

Setting $\varepsilon = \dfrac{\sqrt{D_N^1 + D_N^2 + M}\log^3(N)}{N^4}$, there exists a constant $C > 0$ such that
\begin{equation}\label{eq:Teps2}
    \log\left(\mathcal{N}(\varepsilon, \mathcal{S}_{D_N^1}^{K} \times \widetilde{\mathcal{S}}_{D_N^2}, \|.\|_{\infty})\right) \leq C\left(D_N^1 + D_N^2\right)\log(N).
\end{equation}

Thus, we obtain from Equations~\eqref{eq:Teps1} and \eqref{eq:Teps2},
\begin{multline*}
    \E\left[\w{T}^{\varepsilon}_{1,N}\right]\\
    \leq \dfrac{C_1}{C_2N}\sum_{D_N^1, D_N^2 \in \Xi_N}{\exp\left(-C_2N\mathrm{pen}(D_N^1, D_N^2)\right) \exp\left(C\left(D_N^1 + D_N^2\right)\log(N) - C_2N\mathrm{pen}\left(D_N^1, D_N^2\right)\right)}.
\end{multline*}

We choose the penalty function $(x,y) \mapsto \mathrm{pen}(x, y)$ such that for all $D_N^1, D_N^2 \in \Xi_N$,
\begin{align*}
    C_2N\mathrm{pen}(D_N^1, D_N^2) \geq C(D_N^1 + D_N^2)\log(N) \iff \mathrm{pen}(D_N^1, D_N^2) \geq \frac{C}{C_2}\dfrac{(D_N^1 + D_N^2)\log(N)}{N}.
\end{align*}

Then, for $N$ large enough, we obtain that for all $D_N^1, D_N^2 \in \Xi_N$,
\begin{align*}
    \exp\left(C\left(D_N^1 + D_N^2\right)\log(N) - C_2N\mathrm{pen}\left(D_N^1, D_N^2\right)\right) \leq 1,
\end{align*}

and
\begin{multline}\label{eq:upper-bound-Tepsillon}
    \E\left[\w{T}^{\varepsilon}_{1,N}\right] \leq \frac{C_1}{C_2N}\sum_{D_N^1, D_N^2 \in \Xi_N}{\exp\left(-C_2N\mathrm{pen}(D_N^1, D_N^2)\right)}\\ \leq \dfrac{C_1}{C_2N}\sum_{k=1}^{+\infty}{\sum_{\ell=1}^{+\infty}{\exp\left(-Ck\log(N)\right)\exp\left(-C\ell\log(N)\right)}}\\
    \leq \dfrac{C}{N},
\end{multline}

where the constant $C>0$ depends on $C_1$ and $C_2$. We deduce from Equations~\eqref{eq:upper-bound-T1N} and \eqref{eq:upper-bound-Tepsillon} that there exists a constant $C>0$ such that
\begin{equation}\label{eq:Final-upper-bound-T1N}
    \E\left[\w{T}_{1,N}\right] \leq \dfrac{C}{N}.
\end{equation}

\paragraph{Conclusion.}

We conclude from Equations~\eqref{eq:First-UB-ExcessRisk}, \eqref{eq:UB-T2N} and \eqref{eq:Final-upper-bound-T1N} that
\begin{equation*}
    \E\left[R_2(\w{\bf h}) - R_2({\bf h}^{*})\right] \leq 2\underset{D_N^1, D_N^2 \in \Xi_N}{\inf}{\left\{\E\left[R_2(\widetilde{\bf h}) - R_2({\bf h}^{*})\right] + \mathrm{pen}\left(D_N^1, D_N^2\right)\right\}} + \dfrac{C}{N}.
\end{equation*}
\end{proof}

%%%%%%%%%%%%%%%%%%%%%%%%%%%%%%%%%%%%%%%%%%%%%%%%%%%%%%%%%%%%%%%%%%%%%%%%%%%%%%%%%%%%%%%%
\newpage
%%%%%%%%%%%%%%%%%%%%%%%%%%%%%%%%%%%%%%%%%%%%%%%%%%%%%%%%%%%%%%%%%%%%%%%%%%%%%%%%%%%%%%%%

\bibliographystyle{ScandJStat}
\bibliography{mabiblio.bib}

\end{document}